%% file: main.tex
\documentclass{article}

\usepackage{iclr2027_conference,times}

\input{math_commands.tex}

\usepackage{titletoc}
\usepackage{hyperref}
\usepackage{url}
\usepackage{amssymb}
\usepackage{soul}
\usepackage{mathtools}
\usepackage{float}
\usepackage{booktabs}
\usepackage{tabularx}
\usepackage{array}
\usepackage{colortbl}
\usepackage{longtable}
\usepackage{graphicx}
\usepackage{amsmath}
\usepackage{subcaption}
\usepackage{multirow}
\usepackage{xcolor}
\usepackage{listings}

\usepackage{makecell}
\usepackage{amsthm}

\theoremstyle{definition}
\newtheorem{agsaecondition}{Condition}[section]

\theoremstyle{plain}
\newtheorem{agsaelemma}[agsaecondition]{Lemma}
\newtheorem{agsaeproposition}[agsaecondition]{Proposition}
\newtheorem{agsaecorollary}[agsaecondition]{Corollary}

\lstdefinestyle{semanticprompt}{
    language={},
    basicstyle=\scriptsize\ttfamily,
    columns=fullflexible,
    keepspaces=true,
    breaklines=true,
    breakatwhitespace=true,
    breakindent=0pt,
    showstringspaces=false,
    numbers=none,
    frame=br,
    framerule=0.3pt,
    rulecolor=\color{black!65},
    framesep=8pt,
    xleftmargin=8pt,
    xrightmargin=8pt,
    aboveskip=8pt,
    belowskip=12pt
}

\newsavebox{\agsaepanelA}
\newsavebox{\agsaepanelB}
\newsavebox{\agsaefull}
\newcolumntype{Y}{>{\centering\arraybackslash}X}

\makeatletter
\@ifundefined{AGtopbox}{\newsavebox{\AGtopbox}}{}
\@ifundefined{AGdbox}{\newsavebox{\AGdbox}}{}
\@ifundefined{AGtablewidth}{\newlength{\AGtablewidth}}{}
\makeatother

\title{When Trees Are Not Enough:
Learning Mixed-Topology Feature Graphs with
Adaptive Graph Sparse Autoencoders}

\author{
\textbf{Xiaozuo Shen}$^{1}$,
\textbf{Yifei Cai}$^{2}$,
\textbf{Tian Tan}$^{1}$,
\textbf{Rui Ning}$^{3}$,
\textbf{Chunsheng Xin}$^{2}$,
\textbf{Hongyi Wu}$^{1}$ \\
$^{1}$University of Arizona \\
$^{2}$Iowa State University \\
$^{3}$Old Dominion University \\
Correspondence:
\texttt{xiaozuoshen@arizona.edu},
\texttt{mhwu@arizona.edu}
}

\iclrfinalcopy

\begin{document}

\maketitle
\fancyhead{}
\renewcommand{\headrulewidth}{0pt}

\begin{abstract}

Sparse autoencoders (SAEs) expose interpretable features in large language
model activations, yet existing structured SAEs impose single-parent trees
or forests, while post-hoc graphs permit multiple parents but neither guide
feature learning nor ensure reliable relation recovery. We introduce the
Adaptive Graph Sparse Autoencoder (AG-SAE), a structure-guided training
paradigm that treats each feature's complete parent set as an atomic
structural hypothesis and lets evidence select zero, one, or multiple
parents. By competing complete parent sets against null, subset, and alternative explanations, AG-SAE identifies jointly necessary multi-parent relations while rejecting redundant or spurious alternatives and verifying that each child contributes beyond its parents. The induced topology over SAE features then defines a differentiable structural loss that guides SAE training, while topology-guided refinement mitigates feature absorption and uses persistent reconstruction gaps exposed by the learned structure to initialize new features. The entire graph is then induced again from the revised dictionary by reassessing every feature's complete parent set, closing the dictionary--graph self-consistency cycle.
Experiments demonstrate exact mixed-topology recovery in a controlled toy
model, greater relational reliability and semantic validity than structured
and post-hoc baselines on real LLM activations, and stronger feature-level causal interventions than conventional SAE features. AG-SAE thereby turns recovered mixed-topology feature structure into an unsupervised training signal that improves the dictionary, enables reliable feature organization beyond the topological limitations of trees, and exhibits stronger causal control beyond reconstruction.
\end{abstract}

\section{Introduction}
\vspace{-5pt}

Understanding how large language models (LLMs) represent and organize knowledge requires identifying interpretable features within their entangled high-dimensional activations and recovering the relations among those features. Sparse autoencoders (SAEs) provide such features by decomposing activations into sparse combinations from an overcomplete dictionary, but conventional SAEs learn flat dictionaries that reveal which features exist, not how they relate \citep{bricken2023monosemanticity,huben2024sparse}. As SAE capacity increases, broad concepts often split into finer-grained features \citep{bussmann2025matryoshka}. Hierarchical SAEs capture these levels of abstraction with nested or multiscale dictionaries, but primarily model hierarchy across dictionary scales rather than relations between individual features \citep{zaigrajew2025clip,bussmann2025matryoshka}. Structured SAEs go further by learning feature-level parent--child relations between broad and specialized features \citep{luo2026hsae,cao2026tree}, while typically imposing single-parent trees or forests that assign each child to only one parent.

However, post-hoc analyses of hierarchical SAEs have found multi-parent associations \citep{bussmann2025matryoshka}, and tree-structured SAE studies have acknowledged the limitations of the single-parent constraint \citep{luo2026hsae,cao2026tree}. Taken together, this evidence indicates that a single-parent tree is too restrictive a prior for feature organization. When a child requires multiple parents to explain its representation, any single-parent assignment omits necessary dependencies and distorts its semantic basis (Figure~\ref{fig:ag_sae_overview}); accumulated across a hierarchy, such omissions can mislead feature navigation and mechanistic analysis. In addition, standard metrics for reconstruction, sparsity, and individual feature quality do not assess relational fidelity \citep{gao2025scaling,karvonen2025saebench}. A structured SAE may score well while recovering an incomplete or misleading structure. Good feature recovery therefore does not guarantee reliable structure recovery.

Relaxing the single-parent constraint and adding pairwise edges post hoc merely expands the hypothesis space without ensuring reliable structure recovery. Activation co-occurrence, decoder similarity, local fit gains, redundancy, compositionality, and transitive ancestry can all induce spurious parent--child relations \citep{grandien2026hierarchies}. Multi-parent relations must therefore be identified as complete parent sets rather than assembled from independent edges. Yet moving to complete parent sets remains challenging, because a proposed set may still be displaced by a more parsimonious or equally plausible explanation. A child feature may also be a redundant copy of its parent features, carrying no new or more specific information, while a numerically supported association need not correspond to a semantically coherent broad-to-specific relation. Existing methods remain unable to resolve these ambiguities from LLM activations and recover reliable feature structure without imposing a single-parent topology \citep{bussmann2025matryoshka,luo2026hsae,cao2026tree}.

\begin{figure}[t]
    \centering
    \includegraphics[width=\linewidth]{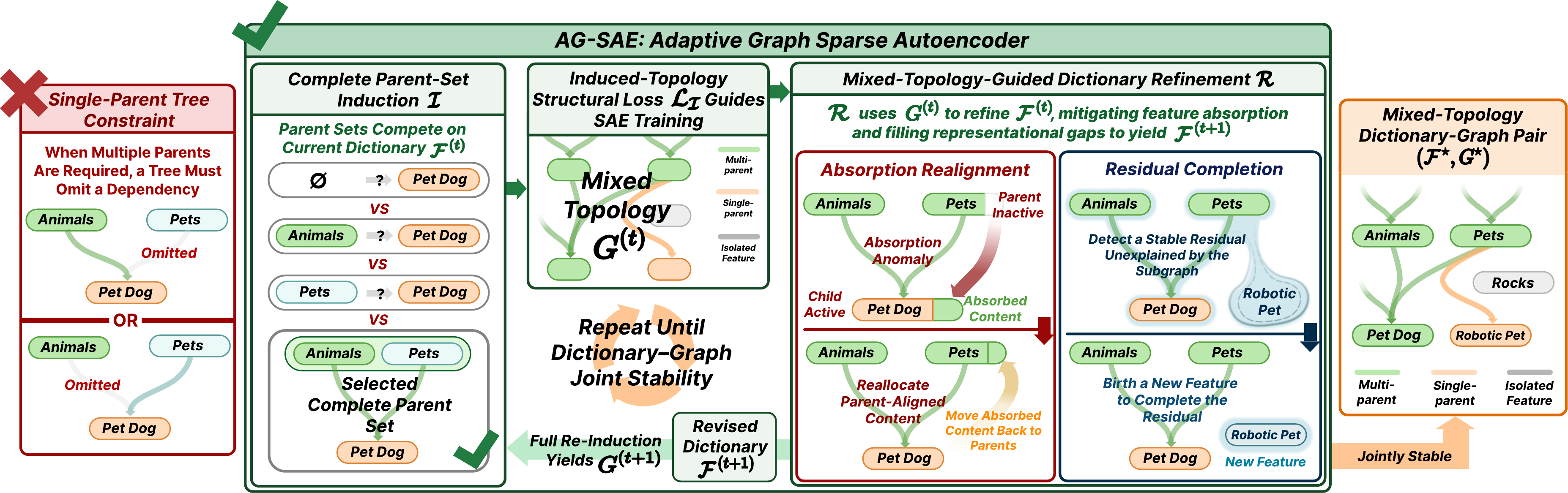}
    \vspace{-15pt}
    \caption{
    \textbf{Overview of AG-SAE.}
    The induced topology guides SAE training and dictionary refinement;
Full Re-Induction closes the cycle, which repeats until the
dictionary--graph pair stabilizes. This example is illustrative.
}
    \label{fig:ag_sae_overview}
    \vspace{-15pt}
\end{figure}

We therefore introduce the Adaptive Graph Sparse Autoencoder (AG-SAE), a structure-guided SAE training paradigm that treats feature structure as both the object to be recovered and an unsupervised signal for feature formation. AG-SAE couples dictionary and graph learning through a dictionary--graph self-consistency cycle. Given the current dictionary, it assigns each child a parent set by competing complete parent sets against null, subset, and alternative explanations, allowing isolated features, single-parent branches, and multi-parent refinements to emerge without a fixed topology. This distinguishes jointly necessary multi-parent refinements from redundant or spurious relations and verifies that each child retains a nonredundant contribution beyond its parents. The induced topology then guides feature-specific encoder and decoder updates, continuously reallocating parent-aligned content and using persistent reconstruction residuals exposed by the learned structure to seed new features where representational gaps remain, thereby expanding the dictionary. After each update, AG-SAE globally re-induces the feature graph under the revised dictionary, closing the dictionary--graph self-consistency cycle. Figure~\ref{fig:ag_sae_overview} summarizes this structure-guided refinement cycle. Experiments show that AG-SAE recovers feature structure with greater relational and semantic fidelity than structured SAE and post-hoc baselines (Section~\ref{sec:reliable-structure}). Beyond these structural gains, we further find that AG-SAE exhibits a non-trivial capability; across both feature addition and feature ablation~\citep{templeton2024scaling,marks2025sparse}, features learned under this paradigm support substantially more effective causal interventions than conventionally trained SAE features (Section~\ref{sec:causal-intervention}). \textbf{Taken together, AG-SAE not only learns a reliable mixed-topology feature graph from SAE representations, but also turns that structure into a training signal that improves the dictionary and enables capabilities beyond reconstruction.} Our contributions are as follows:

\vspace{-5pt}

\begin{enumerate}

\item We formulate feature structure recovery as a complete parent-set assignment problem in which each feature may have zero, one, or multiple parents. Comparing each complete parent set against null, subset, and alternative explanations distinguishes jointly necessary multi-parent refinements from redundant or spurious relations and separates inherited parent content from child-specific contributions. The resulting evidence determines each feature's parent count, extending structure recovery beyond single-parent trees.

\item We introduce AG-SAE, a structure-guided SAE training paradigm that uses recovered feature structure as an unsupervised signal for feature formation. Within its dictionary--graph self-consistency cycle, the induced topology guides feature-specific encoder and decoder updates that continuously reallocate parent-aligned content, while persistent graph-conditioned residuals seed new features that fill representational gaps and expand the dictionary. AG-SAE then globally re-induces all parent sets under the revised dictionary, coupling dictionary optimization to complete structural reassessment.

\item Through experiments, we show that
\textbf{(a)} AG-SAE achieves reliable mixed-topology recovery,
exactly matching all ground-truth complete parent sets in the
toy model and producing graphs with greater relational and
semantic fidelity than structured SAE and post-hoc graph baselines
on real activations. A concrete case further illustrates joint
parent support and nonredundant child content;
\textbf{(b)} AG-SAE's Complete Parent-Set Induction improves
multi-parent structural reliability, and its Absorption Realignment
substantially reduces feature absorption;
and \textbf{(c)} in both feature addition and feature ablation,
AG-SAE features support substantially more effective causal
interventions than conventionally trained SAE features.

\end{enumerate}

\vspace{-13pt}

\section{Related Work and Positioning}
\label{sec:related_work_positioning}
\vspace{-7pt}

Structured SAEs explicitly learn relations between individual features, but restrict each child to a single parent
\citep{muchane2025hierarchical,luo2026hsae,cao2026tree}. Post-hoc analyses
reveal multi-parent associations, yet co-activation and redundancy can
produce spurious relations
\citep{bussmann2025matryoshka,grandien2026hierarchies}. The gap is extending
SAE training beyond single-parent trees while reliably distinguishing
jointly necessary parent sets from redundant or spurious associations.
AG-SAE addresses this gap by evaluating complete parent sets against null,
subset, and alternative explanations. The induced topology then guides
dictionary training and refinement, followed by full re-induction under the
revised dictionary. Appendix~\ref{app:related_work} provides the full
discussion.

\vspace{-10pt}

\section{Adaptive Graph Sparse Autoencoder}

\vspace{-7pt}

The Adaptive Graph Sparse Autoencoder (AG-SAE) casts the learning of
feature structure as a self-consistency problem, jointly developing an SAE feature dictionary and a mixed-topology graph
over the relations among its features through iterative refinement cycles.

Given an activation domain $\mathcal X \subseteq \mathbb R^d$ associated
with a fixed LLM layer, let $\mathbf x \in \mathcal X$ denote an activation
vector drawn from distribution $\mathcal D$ and used as an SAE input during
dictionary training. For each feature $i$, $z_i(\mathbf x)$ denotes its
activation coefficient, and $\mathbf d_i \in \mathbb R^d$ denotes its
unit-norm decoder direction, with $\|\mathbf d_i\|_2=1$. Let
$\mathbf b \in \mathbb R^d$ denote the reconstruction bias. For an SAE
with $M$ features, the reconstruction is written as
$
\widehat{\mathbf x}
=
\sum_{i=1}^{M}
z_i(\mathbf x)\mathbf d_i
+
\mathbf b .
$
For notational simplicity, we omit the reconstruction bias from the
subsequent algorithmic development; see
Appendix~\ref{app:additional_method_details}.

At refinement cycle $t$, let $M_t$ denote the number of current
features. We use $c$ and $p$ to index child and candidate-parent features,
respectively. For a given child $c$, a candidate parent set
$P\subseteq\{1,\ldots,M_t\}\setminus\{c\}$ may satisfy
$P=\varnothing$ or contain one or multiple direct parents, with each
$p\in P$ denoting one candidate parent.
Let $P_c^{(t)}$ denote the parent set ultimately assigned to feature $c$
at refinement cycle $t$.
We then write the global SAE dictionary state and the corresponding
mixed-topology graph state as
{
    \setlength{\abovedisplayskip}{-1pt}
    \setlength{\abovedisplayshortskip}{-1pt}
    \setlength{\belowdisplayskip}{0pt}
    \setlength{\belowdisplayshortskip}{0pt}
\[
\mathcal F^{(t)}
\coloneqq
\left\{
\bigl(z_i^{(t)}(\cdot),\mathbf d_i^{(t)}\bigr)
\right\}_{i=1}^{M_t},
\qquad
G^{(t)}
\coloneqq
\left(P_c^{(t)}\right)_{c=1}^{M_t}.
\]
}Thus, $\mathcal F^{(t)}$ specifies the current SAE dictionary, while
$G^{(t)}$ specifies the relational organization among its features.
\textbf{As illustrated in Figure~\ref{fig:ag_sae_overview}, the AG-SAE refinement cycle proceeds as follows:}
\textbf{(i) Complete Parent-Set Induction $\mathcal I$.}
Starting from the current dictionary $\mathcal F^{(t)}$, $\mathcal I$ jointly evaluates complete
parent sets rather than independent parent--child relations, allowing
zero-, single-, and multi-parent structures to compete and yielding
$G^{(t)}$ (Section~\ref{sec:complete-parent-set-induction}).
\textbf{(ii) Structure-Guided SAE Training.}
In addition to the conventional SAE reconstruction objective
$\mathcal L_{\mathrm{rec}}$~\citep{bricken2023monosemanticity}, AG-SAE
introduces a graph-guided structural loss $\mathcal L_{\mathcal I}$, the
differentiable training counterpart of the relational criteria used by
Complete Parent-Set Induction $\mathcal I$. Starting from
$\mathcal F^{(t)}$ and using $G^{(t)}$ as the structural reference
for this update, the dictionary is optimized with
{
    \setlength{\abovedisplayskip}{0pt}
    \setlength{\abovedisplayshortskip}{0pt}
    \setlength{\belowdisplayskip}{0pt}
    \setlength{\belowdisplayshortskip}{0pt}
\begin{equation}
\label{eq:joint-objective}
\mathcal L_{\mathrm{rec}}\!\left(\mathcal F^{(t)}\right)
+
\mathcal L_{\mathcal I}\!\left(\mathcal F^{(t)};G^{(t)}\right).
\end{equation}
}The structural loss encourages the learned representations to better
support the relations selected by $\mathcal I$, thereby feeding the
induced structure back into dictionary optimization and driving the
dictionary--graph pair toward self-consistency. For conciseness and to avoid redundancy, the underlying relational
criteria are detailed in Section~\ref{sec:complete-parent-set-induction},
with the differentiable formulation provided in
Appendix~\ref{app:structural_training_loss}.
\textbf{(iii) Mixed-Topology-Guided Dictionary Refinement $\mathcal R$.}
Beyond the global structure-guided training above, $\mathcal R$ further
refines the dictionary under the guidance of $G^{(t)}$ by realigning
absorbed parent-related representation and filling persistent
representational gaps, producing $\mathcal F^{(t+1)}$
(Section~\ref{sec:mixed-topology-refinement}).
\textbf{(iv) Full Re-Induction.}
Finally, reapplying $\mathcal I$ to $\mathcal F^{(t+1)}$ yields the next
graph $G^{(t+1)}$, and the cycle repeats until the dictionary--graph pair
stabilizes (Section~\ref{sec:full-reinduction-stability}). 
The overall AG-SAE refinement cycle is therefore summarized as
{
    \setlength{\abovedisplayskip}{1pt}
    \setlength{\abovedisplayshortskip}{1pt}
    \setlength{\belowdisplayskip}{0pt}
    \setlength{\belowdisplayshortskip}{0pt}
    \large
    \[
    \mathcal F^{(t)}
    \xrightarrow{\mathcal I}
    G^{(t)}
    \;\xRightarrow{\;\mathcal L_{\mathrm{rec}}+\mathcal L_{\mathcal I}
    \;\rightarrow\;\mathcal R\;}
    \mathcal F^{(t+1)}
    \xrightarrow{\mathcal I}
    G^{(t+1)}.
    \]
}

\vspace{-15pt}

\subsection{Complete Parent-Set Induction $\mathcal I$}
\label{sec:complete-parent-set-induction}

\vspace{-5pt}

Complete Parent-Set Induction $\mathcal I$ induces the graph $G^{(t)}$
from the current dictionary $\mathcal F^{(t)}$ by jointly evaluating each
candidate parent set $P$, rather than deciding each parent--child relation
separately, which empirically improves structural reliability
(Section~\ref{sec:ablation}). We omit the cycle superscript $(t)$ below. The indicator $a_i(\mathbf x)$ records whether
feature $i$ is active, while $a_P(\mathbf x)$ indicates whether all
features in $P$ are jointly active:
{
    \setlength{\abovedisplayskip}{0pt}
    \setlength{\abovedisplayshortskip}{0pt}
    \setlength{\belowdisplayskip}{0pt}
    \setlength{\belowdisplayshortskip}{0pt}
    \[
    a_i(\mathbf x)
    \coloneqq
    \mathbf 1\!\left\{z_i(\mathbf x)>0\right\},
    \qquad
    a_P(\mathbf x)
    \coloneqq
    \prod_{p\in P}a_p(\mathbf x).
    \]
}\textbf{Parent-Set Support for the Child.} To ensure that the candidate parent set $P$ sufficiently covers the activation regime of child $c$, AG-SAE requires
$\Pr(a_P=1\mid a_c=1)\geq\tau_{\mathrm{cov}}$. Beyond activation coverage, AG-SAE requires the candidate parents to
jointly match the child's representation, measured by
$\mathcal S_{\mathrm{rep}}(c,P)$, which compares the child representation
$z_c(\mathbf x)\mathbf d_c$, combining its direction and activation
strength, with the joint parent representation
$\sum_{p\in P} z_p(\mathbf x)\mathbf d_p$: 
{
    \setlength{\abovedisplayskip}{0pt}
    \setlength{\abovedisplayshortskip}{0pt}
    \setlength{\belowdisplayskip}{0pt}
    \setlength{\belowdisplayshortskip}{0pt}
    \begin{equation}
    \label{eq:parent-support}
    \mathcal S_{\mathrm{rep}}(c,P)
    \coloneqq
    1-
    \frac{
    \mathbb E_{\mathbf x\sim\mathcal D}
    \!\left[
    \left\|
    z_c(\mathbf x)\mathbf d_c
    -
    \sum_{p\in P}z_p(\mathbf x)\mathbf d_p
    \right\|_2^2
    \,\middle|\,
    a_c(\mathbf x)=1
    \right]
    }{
    \mathbb E_{\mathbf x\sim\mathcal D}
    \!\left[
    \left\|
    z_c(\mathbf x)\mathbf d_c
    \right\|_2^2
    \,\middle|\,
    a_c(\mathbf x)=1
    \right]
    }
    \geq
    \tau_{\mathrm{rep}}.
    \end{equation}
}A higher score $\mathcal S_{\mathrm{rep}}(c,P)$ indicates that the candidate parent set more closely matches the child's representation. Since $\mathcal S_{\mathrm{rep}}$ evaluates the candidate set $P$ jointly, parent sets with different cardinalities can compete under the same criterion. A candidate $P$ is retained only if its score exceeds those of the competing sets by a margin. This enables AG-SAE to extend beyond the single-parent restriction while recovering more reliable structure than independently assembled parent--child relations (Section~\ref{sec:ablation}). Details are provided in Appendix~\ref{app:complete_parent_set_induction}.

\textbf{Child Contribution Beyond Parents.} We further require $c$ to provide reconstruction contribution beyond its
complete parent set. Let $\mathcal E_{c,P}(S)$ denote the reconstruction error from feature set $S$, evaluated on samples where $c$ and all features in
$P$ are jointly active:
{
    \setlength{\abovedisplayskip}{0pt}
    \setlength{\abovedisplayshortskip}{0pt}
    \setlength{\belowdisplayskip}{1pt}
    \setlength{\belowdisplayshortskip}{1pt}
    \begin{equation}
    \label{eq:child-innovation}
    \begin{aligned}
    \mathcal E_{c,P}(S)
    &\coloneqq
    \min_{\boldsymbol\beta\geq\mathbf 0}
    \mathbb E_{\mathbf x\sim\mathcal D}
    \!\left[
    \left\|
    \mathbf x-
    \sum_{i\in S}\beta_i z_i(\mathbf x)\mathbf d_i
    \right\|_2^2
    \,\middle|\,
    a_c(\mathbf x)a_P(\mathbf x)=1
    \right],
    \\
    \mathcal S_{\mathrm{inn}}(c,P)
    &\coloneqq
    \frac{
    \mathcal E_{c,P}(P)
    -
    \mathcal E_{c,P}(P\cup\{c\})
    }{
    \mathcal E_{c,P}(\varnothing)
    }
    \geq
    \tau_{\mathrm{inn}}.
    \end{aligned}
    \end{equation}
}The score $\mathcal S_{\mathrm{inn}}(c,P)$ measures the
reduction in reconstruction error when the child $c$ is added to its
parent set, comparing reconstruction with $P$ alone against
$P\cup\{c\}$. The denominator $\mathcal E_{c,P}(\varnothing)$ normalizes
this gain by the zero-reconstruction error. A higher score therefore
indicates that the child provides additional representational value beyond
$P$. Together with parent-set support, this helps reject
redundant or spurious parent--child relations. Fitting details are provided in
Appendix~\ref{app:complete_parent_set_induction}.
Together, these criteria allow $\mathcal I$ to select one parent set per
child under global acyclicity, yielding $G=\mathcal I(\mathcal F)$.
The resulting graph, together with these criteria, defines the structural
signal $\mathcal L_{\mathcal I}$ that guides dictionary optimization (Appendix~\ref{app:structural_training_loss}).

\vspace{-9pt}

\subsection{Mixed-Topology-Guided Dictionary Refinement $\mathcal R$}
\label{sec:mixed-topology-refinement}

\vspace{-6pt}

The operator $\mathcal R$ refines the feature dictionary using induced
topology $G$. Learned parent--child relations guide
Absorption Realignment to redistribute misallocated parent-related
representation, while persistent graph-conditioned residuals guide
Residual Completion to fill representational gaps.

\vspace{-7pt}

\paragraph{Absorption Realignment.} Feature absorption can cause a child to carry representation aligned with
its broader parents~\citep{chanin2025absorption}. Using the parent set $P_c$ recovered by $\mathcal I$, we decompose the decoder direction of each child
$c$ with $P_c\neq\varnothing$ as
$\mathbf d_c=\sum_{p\in P_c}\delta_{cp}\mathbf d_p+\mathbf r_c$,
where each $\delta_{cp}\geq0$ measures how strongly the child direction
$\mathbf d_c$ aligns with parent direction $\mathbf d_p$, while
$\mathbf r_c$ represents the child-specific direction retained after
separating the parent-aligned components. Collecting these
parent-alignment coefficients as
$\boldsymbol\delta_c=(\delta_{cp})_{p\in P_c}$, we estimate them jointly
through nonnegative ridge fitting, where $\lambda_\delta>0$ controls
the ridge regularization strength:
{
    \setlength{\abovedisplayskip}{1pt}
    \setlength{\abovedisplayshortskip}{1pt}
    \setlength{\belowdisplayskip}{1pt}
    \setlength{\belowdisplayshortskip}{1pt}
    \begin{equation}
    \label{eq:allocation-fit}
    \boldsymbol{\delta}_c
    \coloneqq
    \arg\min_{\boldsymbol{\delta}_c\geq\mathbf 0}
    \left[
    \left\|
    \mathbf d_c
    -
    \sum_{p\in P_c}
    \delta_{cp}\mathbf d_p
    \right\|_2^2
    +
    \lambda_\delta
    \left\|\boldsymbol{\delta}_c\right\|_2^2
    \right],
    \qquad
    \mathbf r_c
    \coloneqq
    \mathbf d_c
    -
    \sum_{p\in P_c}
    \delta_{cp}\mathbf d_p .
    \end{equation}
}The decomposition above identifies which parts of the child direction
align with its parents and which remain specific to the child.
Absorption Realignment uses this structure to redistribute representation
within the local parent--child subgraph during training: parent-aligned
content is routed along the corresponding parent directions, while child-specific
content remains with the child. The same local contribution can be written without changing the reconstruction:
{
    \setlength{\abovedisplayskip}{2pt}
    \setlength{\abovedisplayshortskip}{2pt}
    \setlength{\belowdisplayskip}{0pt}
    \setlength{\belowdisplayshortskip}{0pt}
    \[
    z_c(\mathbf x)\mathbf d_c
    +\sum_{p\in P_c}z_p(\mathbf x)\mathbf d_p
    =
    \|\mathbf r_c\|_2 z_c(\mathbf x)
    \frac{\mathbf r_c}{\|\mathbf r_c\|_2}
    +\sum_{p\in P_c}
    \left(
    z_p(\mathbf x)+\delta_{cp}z_c(\mathbf x)
    \right)\mathbf d_p.
    \]
}Under this decomposition, Absorption Realignment routes parent-aligned
content out of the child and back along the corresponding parent directions,
while continuously adapting the child--parent allocation during training.
This directly mitigates local feature absorption (Section~\ref{sec:ablation}). Implementation details
are provided in
Appendix~\ref{app:mixed_topology_dictionary_optimization}.

\vspace{-7pt}

\paragraph{Residual Completion.}
Absorption Realignment mitigates absorption among related
features, whereas Residual Completion fills persistent representational
gaps. Let $\mathbf e(\mathbf x)\coloneqq
\mathbf x-\widehat{\mathbf x}$ denote the reconstruction residual of the
current SAE. When child $c$ and all features in $P_c$ are active, we define
$\mathbf e_c(\mathbf x)\coloneqq
a_c(\mathbf x)a_{P_c}(\mathbf x)\mathbf e(\mathbf x)$ to retain the
residual associated with this parent--child relation.
A stable direction in $\mathbf e_c$ indicates a persistent
representational gap. Residual Completion reuses a matching existing
feature when possible; otherwise, it learns a new sparse feature through:
{
    \setlength{\abovedisplayskip}{0pt}
    \setlength{\abovedisplayshortskip}{0pt}
    \setlength{\belowdisplayskip}{0pt}
    \setlength{\belowdisplayshortskip}{0pt}
    \begin{equation}
    \label{eq:residual-completion}
    \min_{\substack{z\geq0\\\|\mathbf d\|_2=1}}
    \mathbb E_{\mathbf x\sim\mathcal D}
    \left[
    \left\|
    \mathbf e_c(\mathbf x)
    -
    z(\mathbf x)\mathbf d
    \right\|_2^2
    +
    \lambda_{\mathrm{sp}}z(\mathbf x)
    \right],
    \end{equation}
}where $\lambda_{\mathrm{sp}}>0$ controls activation sparsity. The new feature captures the persistent residual component exposed by the learned
relation, thereby expanding the dictionary to better fill a previously unexplained
representational gap. Implementation details are provided in
Appendix~\ref{app:mixed_topology_dictionary_optimization}.

\vspace{-8pt}

\subsection{Full Re-Induction and Joint Stability}
\label{sec:full-reinduction-stability}

\vspace{-6pt}

After $\mathcal R$ uses the induced graph $G^{(t)}$ to refine the
dictionary into $\mathcal F^{(t+1)}$, Full Re-Induction closes the
refinement cycle by reapplying $\mathcal I$ to the updated dictionary. Each
feature therefore reassesses its parent set under the updated
representations, yielding
{
    \setlength{\abovedisplayskip}{1pt}
    \setlength{\abovedisplayshortskip}{1pt}
    \setlength{\belowdisplayskip}{1pt}
    \setlength{\belowdisplayshortskip}{1pt}
    \[
    G^{(t+1)}
    =
    \mathcal I\!\left(\mathcal F^{(t+1)}\right)
    =
    \left(P_c^{(t+1)}\right)_{c=1}^{M_{t+1}}.
    \]
}AG-SAE terminates when both the graph and feature contributions
stabilize across a refinement cycle. We measure the changes between
successive cycles by $d_G$ for the graph and $d_F$ for the feature
contributions on held-out activations. For tolerances
$\gamma_G,\gamma_F>0$, AG-SAE stops when
{
    \setlength{\abovedisplayskip}{1pt}
    \setlength{\abovedisplayshortskip}{1pt}
    \setlength{\belowdisplayskip}{1pt}
    \setlength{\belowdisplayshortskip}{1pt}
    \[
    d_G\!\left(G^{(t+1)},G^{(t)}\right)
    \leq
    \gamma_G,
    \qquad
    d_F\!\left(
    \mathcal F^{(t+1)},\mathcal F^{(t)}
    \right)
    \leq
    \gamma_F.
    \]
}Appendix~\ref{app:full_reinduction_stability} gives estimator details, while Figure~\ref{fig:graph_refinement_dynamics} shows the trajectories of $d_F$ and $d_G$ across refinement cycles. We also provide a theoretical analysis of dictionary--graph stability in Appendix~\ref{app:stability_theory}. When both conditions hold, AG-SAE returns the final dictionary--graph pair
{
    \setlength{\abovedisplayskip}{2pt}
    \setlength{\abovedisplayshortskip}{2pt}
    \setlength{\belowdisplayskip}{0pt}
    \setlength{\belowdisplayshortskip}{0pt}
    \[
    (\mathcal F^\star,G^\star)
    =
    (\mathcal F^{(t+1)},G^{(t+1)}).
\]
}

\vspace{-12pt}

\section{Experimental Evaluation}
\vspace{-10pt}

We evaluate AG-SAE by \textbf{(a)} measuring its ability to recover structure on a mixed-topology toy model with known ground-truth structure; \textbf{(b)} assessing out-of-sample relational reliability and
blinded semantic validity on real LLM activations, complemented by a feature-level case study; \textbf{(c)} further demonstrating that features learned under the AG-SAE paradigm exhibit stronger causal intervention capabilities; and \textbf{(d)} ablating its core operators to isolate their contributions.

\vspace{-10pt}

\subsection{Experimental Setup}
\vspace{-5pt}

\textbf{Models, data, and baselines.}
SAE baselines use BatchTopK activations \citep{bussmann2024batchtopk}. Our main setting studies layer-13 residual-stream features of Gemma-2-2B on MiniPile \citep{gemmateam2024gemma2,kaddour2023minipile}. We compare AG-SAE with Vanilla SAE Post-hoc, a strong post-hoc baseline using the same mixed-topology search space, Matryoshka SAE with scaled-MCS DAG and best-MCS tree projections, Tree SAE, and HSAE \citep{huben2024sparse,bussmann2025matryoshka,cao2026tree,luo2026hsae}. Toy-model experiments follow SynthSAEBench and Matryoshka SAE \citep{chanin2026synthsaebench,bussmann2025matryoshka}, with cross-model and cross-domain evaluations on Qwen3.5-2B-Base and PubMed \citep{qwen2026qwen35,nlm2025pubmed}. Feature labeling and blinded relation evaluation use Qwen3-30B-A3B under SAEBench AutoInterp and HSAE-style protocols \citep{qwen2025qwen3,karvonen2025saebench,luo2026hsae}.

\textbf{Training and evaluation protocol.}
Hierarchical methods and AG-SAE use matched four-scale dictionaries, with matched training steps and data budgets and comparable sparsity operating points. Structure induction and threshold validation use disjoint data; the resulting dictionary, graph, and thresholds are then frozen for a separate report-only evaluation that cannot affect structure construction, model selection, or dictionary adaptation. Causal interventions evaluate bidirectional feature addition and ablation on 416 held-out prompts spanning 26 first-letter concepts. Complete dictionary widths, layer choices, training specifications, baseline implementations, cross-setting results, and experiment-specific hyperparameters are provided in Appendices~\ref{app:additional_results} and~\ref{app:experimental_setup}.

\begin{table*}[!t]
\centering
\caption{
\textbf{AG-SAE structure recovery in toy and real LLM settings.}
\textbf{A.} Dictionary quality and structure recovery on the ground-truth
mixed-topology toy model. In A, \(\dagger\) denotes N/A for methods without
a native parent graph; Vanilla SAE's eight no-parent matches reflect
the default empty parent set.
\textbf{B.} Final AG-SAE feature graph on Gemma-2-2B layer-13 activations
from MiniPile. Panels C--D use the same setting. Category percentages use the \(30{,}718\) active features
(\(99.99\%\) of \(30{,}720\)); Multi / parented is the fraction of parented
features with multiple parents.
\textbf{C.} Dictionary and graph statistics across refinement cycles.
\textbf{D.} Representation quality and relation validation across methods
(Abs., Feature Absorption; AutoInt., AutoInterp).
PSV, NR, and SV denote Predictive Structural Validity, Non-Redundancy,
and Semantic Validity; Joint denotes their intersection.
\(\dagger\) marks Tree SAE's lower \(L_0\) from its native sparsification,
which also gives its Abs.\ score a sparsity advantage.
}
\label{tab:ag-sae-combined}
\label{tab:real-llm-validation}

\vspace{-3pt}

\begingroup
\footnotesize

\setlength{\tabcolsep}{2.5pt}
\setlength{\arrayrulewidth}{0.6pt}
\setlength{\aboverulesep}{0.8pt}
\setlength{\belowrulesep}{0.8pt}
\renewcommand{\arraystretch}{1.15}

\def\agsaerow{\rule[-0.75ex]{0pt}{3.0ex}}

\def\AGDheadstrut{\rule[-1.1ex]{0pt}{3.8ex}}


\sbox{\AGtopbox}{%
\setlength{\arrayrulewidth}{\heavyrulewidth}%
\begin{tabular}{
    @{}c
    @{\hspace{5pt}}
    !{\vrule width 0.6pt}
    @{\hspace{5pt}}
    c@{}
}
\Xhline{1.5pt}

\begin{tabular}[t]{
    @{}l
    !{\vrule width 0.45pt}
    cc
    !{\vrule width 0.45pt}
    cc@{}
}

\multicolumn{5}{c}{
    \agsaerow\textbf{A. Toy-model evaluation}
}
\\[2pt]

\midrule

\multicolumn{1}{c}{\agsaerow}
&
\multicolumn{2}{c}{\textbf{Dictionary quality}}
&
\multicolumn{2}{c}{\textbf{Structure recovery}}
\\

\specialrule{\lightrulewidth}{0pt}{0pt}

\agsaerow\textbf{Method}
& \textbf{Test}
& \textbf{Feat.}
& \textbf{Exact PS /24} \(\uparrow\)
& \textbf{HN Rejected}
\\[-1pt]

\agsaerow
& \(R^2\) \(\uparrow\)
& \textbf{/24} \(\uparrow\)
& \textbf{0/1/2-Parent}
& \textbf{/32} \(\uparrow\)
\\

\midrule

\agsaerow Vanilla SAE
& 1.00
& \textbf{24}
& \(8,(8/0/0)\)
& \textemdash\(\dagger\)
\\

\agsaerow Matryoshka SAE
& 0.92
& 19
& \(8,(8/0/0)\)
& \textemdash\(\dagger\)
\\

\agsaerow Tree SAE
& 0.91
& 10
& \(8,(6/2/0)\)
& 28
\\

\agsaerow HSAE
& 0.99
& 23
& \(10,(8/2/0)\)
& 25
\\

\agsaerow\textbf{AG-SAE (ours)}
& \textbf{1.00}
& \textbf{24}
& \(\mathbf{24},(\mathbf{8}/\mathbf{8}/\mathbf{8})\)
& \textbf{32}
\\

\end{tabular}
&
\begin{tabular}[t]{
    @{}lr
    @{\hspace{5pt}}
    !{\vrule width 0.6pt}
    @{\hspace{5pt}}
    l
    r@{\hspace{8pt}}
    r@{\hspace{8pt}}
    r@{}
}

\multicolumn{2}{
    c@{\hspace{5pt}}
    !{\vrule width 0.6pt}
    @{\hspace{5pt}}
}{
    \agsaerow\textbf{B. Final mixed-topology graph}
}
&
\multicolumn{4}{c}{
    \textbf{C. Refinement across cycles}
}
\\[2pt]

\midrule

\agsaerow\textbf{Metric}
& \textbf{Result}
& \textbf{Metric}
& \(G^{(0)}\)
& \(G^{(2)}\)
& \(G^{(5)}\)
\\

\midrule

\agsaerow Total features
& 30,720
& Single-parent
& 2,343
& 5,142
& 5,421
\\

\agsaerow Isolated
& 19,912~\textbf{(64.82\%)}
& Multi-parent
& 782
& 1,637
& 1,670
\\

\agsaerow Roots
& 3,715~\textbf{(12.09\%)}
& Absorption realignment
& \textemdash
& 5,973
& 7,068
\\

\agsaerow Single-parent
& 5,421~\textbf{(17.65\%)}
& Residual completion
& \textemdash
& 161
& 12
\\

\agsaerow Multi-parent
& 1,670~\textbf{(5.44\%)}
& Relation gain \(\uparrow\)
& 0.0456
& 0.0875
& 0.0912
\\

\agsaerow Multi / parented
& \textbf{23.55\%}
& Graph change \(\boldsymbol{d_G}\) (\%) \(\downarrow\)
& \textemdash
& \textbf{18.20\%}
& \textbf{3.06\%}
\\

\agsaerow Longest path
& 8 edges
& Feature change \(\boldsymbol{d_F}\) (\%) \(\downarrow\)
& \textemdash
& \textbf{68.45\%}
& \textbf{9.25\%}
\\

\end{tabular}
\\
\Xhline{1.5pt}
\end{tabular}%
}


\setlength{\arrayrulewidth}{\lightrulewidth}

\def\AGDdividerwidth{0.6pt}
\def\AGDbody{%
\specialrule{1.5pt}{0pt}{0pt}

\multirow{2}{*}{%
    \shortstack[l]{%
        \textbf{D. Real LLM}\\
        \textbf{validation}%
    }%
}
&
\multicolumn{4}{c!{\vrule width \AGDdividerwidth}}{
    \textbf{Representation quality}
}
&
\multicolumn{7}{c}{
    \textbf{Recovered-structure reliability}
}
\\

\cline{2-12}

\AGDheadstrut
&
\multicolumn{4}{c!{\vrule width \AGDdividerwidth}}{
    \textbf{Global dictionary}
}
&
\multicolumn{3}{c!{\vrule width 0.6pt}}{
    \textbf{Single-parent} (\%) \(\uparrow\)
}
&
\multicolumn{3}{c!{\vrule width 0.6pt}}{
    \textbf{Multi-parent} (\%) \(\uparrow\)
}
&
\textbf{Joint} (\%) \(\uparrow\)
\\

\cline{2-12}

\AGDheadstrut\textbf{Method}
& \textbf{EV} \(\uparrow\)
& \(\boldsymbol{L_0}\)
& \textbf{Abs.} \(\downarrow\)
& \textbf{AutoInt.} \(\uparrow\)
& \textbf{PSV}
& \textbf{NR}
& \textbf{SV}
& \textbf{PSV}
& \textbf{NR}
& \textbf{SV}
&
\mbox{%
    \textbf{PSV}\hspace{0.8pt}\(\cap\)\hspace{0.8pt}%
    \textbf{NR}\hspace{0.8pt}\(\cap\)\hspace{0.8pt}%
    \textbf{SV}%
}
\\

\midrule

\textbf{AG-SAE}
& 0.670
& 50.13
& 9.16
& \textbf{86.44}
& \textbf{99.75}
& \textbf{97.00}
& \textbf{76.88}
& \textbf{97.00}
& \textbf{91.88}
& \textbf{75.00}
& \textbf{69.56}
\\

Vanilla SAE Post-hoc
& 0.668
& 50.00
& 51.39
& 80.03
& 71.50
& 75.50
& 69.13
& 32.63
& 65.75
& 58.25
& 40.63
\\

Matryoshka SAE (DAG)
& \textbf{0.680}
& 50.00
& 12.71
& 81.02
& 60.38
& 66.25
& 30.38
& 35.38
& 57.00
& 35.00
& 25.44
\\

Matryoshka SAE (Tree)
& \textbf{0.680}
& 50.00
& 12.71
& 81.02
& 48.25
& 54.50
& 24.25
& \multicolumn{3}{c!{\vrule width 0.6pt}}{N/A}
& 22.38
\\

HSAE
& 0.665
& 50.00
& 10.24
& 79.90
& 92.88
& 92.13
& 56.88
& \multicolumn{3}{c!{\vrule width 0.6pt}}{N/A}
& 51.13
\\

Tree SAE
& 0.663
& \(40.84^{\dagger}\)
& \textbf{5.11}
& 75.14
& 24.63
& 24.13
& 23.25
& \multicolumn{3}{c!{\vrule width 0.6pt}}{N/A}
& 18.75
\\

\specialrule{1.5pt}{0pt}{0pt}
}

\sbox{\AGdbox}{%
\begin{tabular}{
    @{}lcccc
    !{\vrule width \AGDdividerwidth}
    ccc
    !{\vrule width 0.6pt}
    ccc
    !{\vrule width 0.6pt}
    c@{}
}
\AGDbody
\end{tabular}%
}

\setlength{\AGtablewidth}{\wd\AGtopbox}
\ifdim\wd\AGdbox>\AGtablewidth
    \setlength{\AGtablewidth}{\wd\AGdbox}
\fi


\resizebox{\linewidth}{!}{%
\begin{tabular}{@{}c@{}}

\makebox[\AGtablewidth][c]{\usebox{\AGtopbox}}
\\
\noalign{\vskip 8pt}

\begin{tabular*}{\AGtablewidth}{
    @{\extracolsep{\fill}}
    lcccc
    !{\vrule width \AGDdividerwidth}
    ccc
    !{\vrule width 0.6pt}
    ccc
    !{\vrule width 0.6pt}
    c
    @{}
}
\AGDbody
\end{tabular*}

\end{tabular}%
}

\endgroup

\vspace{-10pt}
\end{table*}

\vspace{-7pt}

\subsection{Exact Mixed-Topology Recovery in a Controlled Toy Model}
\vspace{-5pt}

Since no reliable ground-truth structure is available for real LLM
activations, we evaluate structure recovery in a mixed-topology toy
model constructed following the generation procedure of
SynthSAEBench~\citep{chanin2026synthsaebench}. Observations are sparse
linear superpositions of \(24\) predefined feature directions with
positive activation magnitudes. Following the parent-conditioned
activation mechanism of the Matryoshka SAE toy
model~\citep{bussmann2025matryoshka}, a child can activate only when
all its parents are active. The graph contains eight features each
with zero, one, and two parents
(Appendix Figure~\ref{fig:toy_model}).
We also construct \(32\) hard negatives spanning correlation-only
candidates, transitive ancestors, and incomplete subsets of true
multi-parent sets (Appendix~\ref{app:toy_model_setup}).

At matched sparsity, AG-SAE attains \(R^2=1.00\), recovers all
\(24/24\) features and \(24/24\) complete parent sets, and rejects
all \(32/32\) hard negatives
(Table~\ref{tab:ag-sae-combined}, Panel A).
Vanilla SAE also recovers all \(24/24\) features with \(R^2=1.00\);
its lack of a native parent graph leaves only the eight default
no-parent matches. Tree SAE and HSAE each recover only two of the
eight single-parent sets and none of the eight two-parent sets.
These comparisons establish that accurate reconstruction and feature
matching do not determine feature organization. In this controlled
setting, AG-SAE recovers the full mixed topology and rejects all
tested spurious structural explanations.

\vspace{-5pt}

\subsection{Reliable Structure Learning from Real LLM Activations}
\label{sec:reliable-structure}

\vspace{-3pt}

\paragraph{Learned Mixed-Topology Feature Graph.}
On layer-13 Gemma-2-2B activations from MiniPile, AG-SAE learns a feature
graph containing isolated features, roots, single-parent relations, and
multi-parent relations. Panel B of Table~\ref{tab:ag-sae-combined} shows
that multi-parent features comprise \textbf{23.55\%} of all parented
features, establishing a substantial non-tree component in the learned
topology. Tree SAE and HSAE inherit graph depth from four predefined
dictionary levels, while AG-SAE's longest directed path spans eight edges,
or nine feature levels, demonstrating that its joint dictionary--graph
training learns relational depth beyond a prescribed dictionary hierarchy.
Across refinement cycles, held-out relation gain rises from
\(4.56\%\) to \(9.12\%\), while \(d_F\) and \(d_G\) decline
after \(G^{(2)}\) to \(9.25\%\) and \(3.06\%\), respectively
(Table~\ref{tab:ag-sae-combined}, Panel C), supporting progressive
joint stabilization of the dictionary--graph pair as relation gain approaches a plateau.
Appendix~\ref{app:graph_coverage_validity} provides full refinement trajectories
and analysis (Figure~\ref{fig:graph_refinement_dynamics}), while
Appendix~\ref{app:cross_setting_results} reports cross-model and cross-domain results.

\vspace{-4pt}

\paragraph{Relation Reliability Validation.}
With the dictionary, graph, and thresholds frozen, Panel D of
Table~\ref{tab:ag-sae-combined} combines two activation-based structural
diagnostics with blinded semantic validation on a disjoint report-only
split. Predictive Structural Validity (PSV) measures whether recovered
parent--child dependencies remain stable and predictively reliable on
unseen activations. Non-Redundancy (NR) measures the child's additional
held-out reconstruction contribution beyond its parents. Both diagnostics are method-agnostic and use the same held-out criteria for all approaches. Semantic Validity (SV) tests relational coherence in
human-interpretable natural language using frozen semantic labels generated
following the SAEBench protocol~\citep{karvonen2025saebench} and blinded
judgments, neither of which enters induction.
Joint (\(\mathrm{PSV}\cap\mathrm{NR}\cap\mathrm{SV}\)) reports the fraction passing all three criteria on the same fixed, arity-matched relation samples.

AG-SAE leads all six single- and multi-parent reliability measures.
Its single-parent PSV, NR, and SV reach 99.75\%, 97.00\%, and 76.88\%, exceeding the best competing
value per metric by \(6.87\), \(4.87\), and \(7.75\) percentage
points. The gap is decisive for multi-parent relations.
AG-SAE reaches 97.00\% PSV, 91.88\% NR, and 75.00\% SV, leading the strongest baselines reporting each
metric by \(61.62\), \(26.13\), and \(16.75\) points.
Its 69.56\% Joint rate exceeds the full-space
Vanilla SAE Post-hoc comparator by \(28.93\) points.
These establish predictive reliability on unseen activations,
nonredundant child content, and the strongest blinded semantic coherence
among evaluated methods.

At comparable sparsity, AG-SAE records a 0.01\% dead-feature
rate, the highest AutoInterp score of 86.44\%, and the lowest
absorption among methods near \(L_0=50\), with EV within \(0.010\)
of the best reported value. Tree SAE's lower raw absorption at reduced
\(L_0\) accompanies weak relation-validation results, showing that
absorption alone does not determine structural reliability. AG-SAE thus achieves the strongest relation reliability with competitive representation quality.

\vspace{-5pt}

\begin{figure}[!t]
    \centering
    \includegraphics[width=\linewidth]{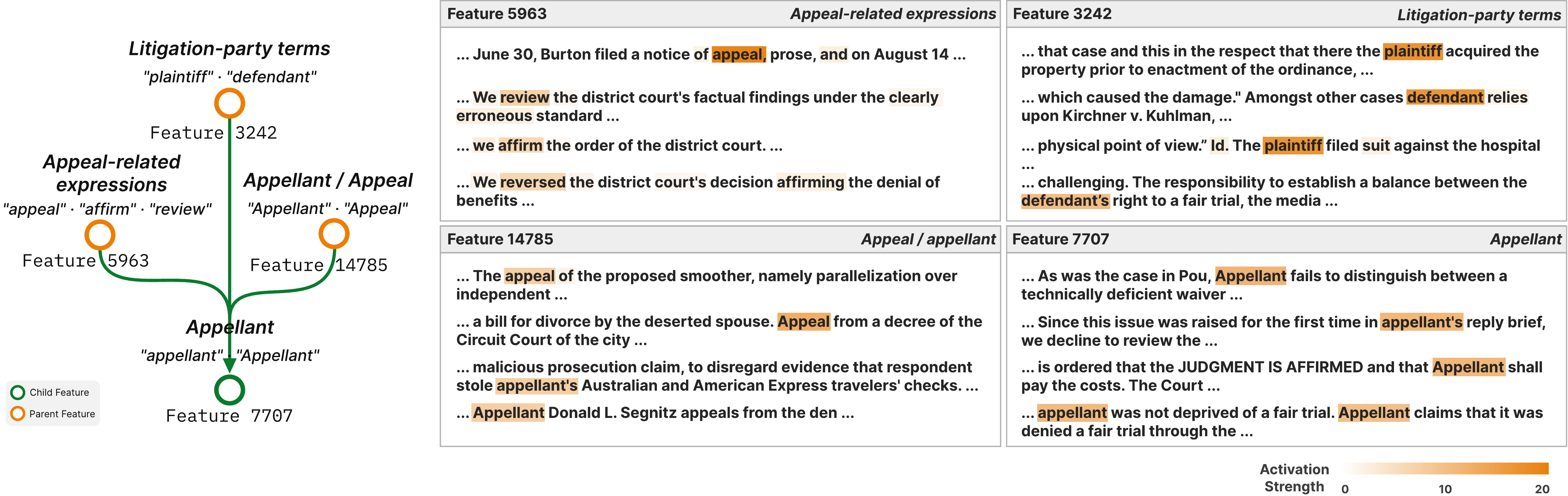}
    \vspace{-15pt}
    \caption{
    \textbf{Learned three-parent relation for appellant references.}
    \textbf{Left.} Feature 7707 and its complete parent set in final
    AG-SAE trained on MiniPile Gemma-2-2B layer-13 activations.
    Feature IDs are global across dictionary banks.
    \textbf{Right.} Activation examples selected per feature from profiling
    data and report contexts.
    Orange shading shows per-token activation on the same raw scale.
    }
    \label{fig:real-case}
    \vspace{-15pt}
\end{figure}

\begin{figure}[!t]
    \centering
    \includegraphics[width=0.9\linewidth]{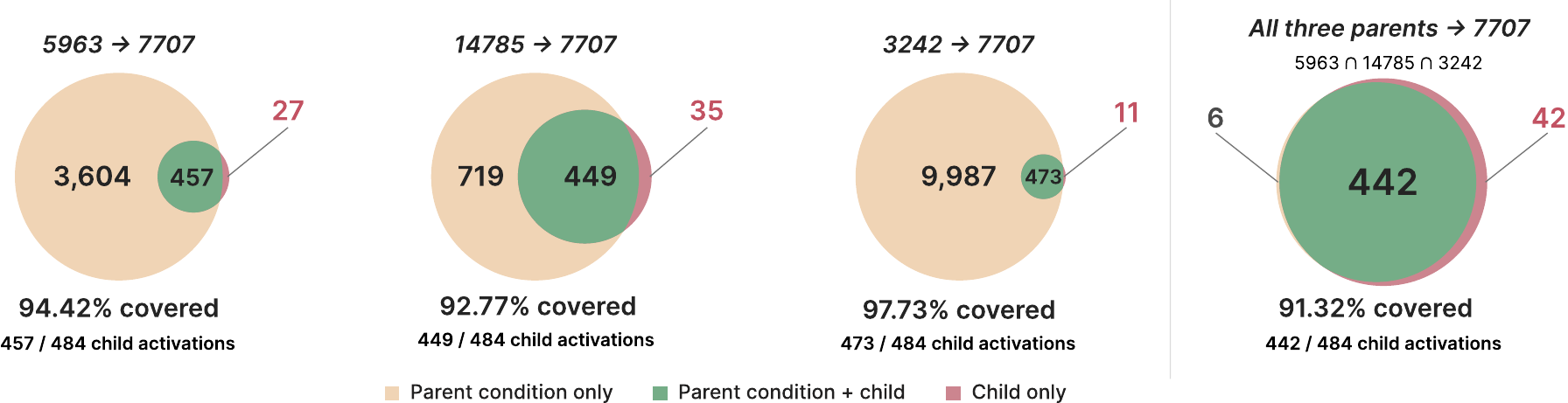}
    \vspace{-5pt}
    \caption{
    \textbf{Individual and joint parent coverage of feature 7707.}
    Counts use \(4{,}910{,}199\) report token positions with frozen
    activation thresholds. The first three panels show individual
    parents; the rightmost requires all three parents to be active.
    Orange denotes parent-only activation, green coactivation with
    the child, and pink child activation with the parent condition
    unmet. Coverage uses the \(484\) child activations as its
    denominator. Panels use separate visual scales.
    }
    \label{fig:case-coverage}
    \vspace{-15pt}
\end{figure}

\paragraph{Case Study of a Learned Multi-Parent Relation.}
To understand how the aggregate reliability results translate into
individual feature organization, we study a learned multi-parent
relation chosen for its readily interpretable legal vocabulary
(Figure~\ref{fig:real-case}). In the
Gemma-2-2B/MiniPile setting, AG-SAE assigns child \(7707\) the complete
parent set \(\{5963,14785,3242\}\). The parents activate on appellate
expressions, appeal/appellant lexical forms, and litigation-party
terms, respectively; the child's displayed activations focus on
appellant references in litigation. These parental activation scopes
overlap and extend beyond the child's, including nonlegal uses of
\emph{appeal} for feature \(14785\).

The complete parent configuration identifies the child's activation
regime with both broad coverage and high specificity
(Figure~\ref{fig:case-coverage}).
Using native SAE activations with the dictionary and thresholds frozen, all three parents are active in
\(442/484=91.32\%\) of child-active positions.
Conversely, the child is active in \(442/448=98.66\%\) of the parents' joint activations, whereas it is active in at most \(38.44\%\) of any individual parent's activations. Their joint activation therefore identifies the child's
more specific response pattern within the parents' broader scopes.

The complete parent set also provides stronger predictive support
than its proper subsets and matched alternatives. Its held-out fit to
the child's decoder contribution reaches \(0.5497\), compared with
\(0.2452\) for the best single parent, \(0.4543\) for the best
two-parent subset, and \(0.4627\) for the strongest evaluated
alternative of the same cardinality. Removing any parent and
refitting the remaining set reduces the score by at least \(0.0954\).
The child additionally contributes a normalized innovation gain
of \(\mathcal S_{\mathrm{inn}}(c,P)=0.06177\) on the \(442\) positions
where it and all parents are active.
All reconstruction coefficients are fitted on a disjoint split
and frozen for report evaluation.
Full comparisons and scoring details appear in
Appendix~\ref{app:multi_parent_case_study}.
\textbf{Together, the aggregate validation and this case study
show that AG-SAE learns reliable multi-parent structure with
interpretable joint parent support and distinct child contributions.}

\vspace{-5pt}

\subsection{AG-SAE Features Enable Stronger Causal Interventions}
\label{sec:causal-intervention}

\vspace{-5pt}

We further find that AG-SAE's structure-guided training yields features with greater causal intervention capacity concentrated in their leading features. Building on SAEBench's first-letter probing setup~\citep{karvonen2025saebench}, we use probes to identify the leading
feature for each of the 26 first-letter concepts and evaluate these features
on 416 disjoint held-out prompts. Following prior SAE intervention
studies~\citep{templeton2024scaling}, feature addition tests concept writing,
whereas feature ablation tests concept erasure. Our experiments
(Figure~\ref{fig:causal_intervention_triptych}) show that \textbf{AG-SAE's leading
feature exerts a stronger average causal effect than its Vanilla SAE
counterpart, with the improvement appearing simultaneously in both concept
writing and erasure for most tested concepts.} Crucially, this advantage is concentrated in the Top-1 feature, demonstrating that AG-SAE
consolidates causal control into a single leading feature. Matched training
counterfactuals further establish this causal advantage as a functional
consequence specific to the AG-SAE paradigm; neither conventional SAE
continuation nor decoder scaling reproduces it
(Appendix~\ref{app:causal_training_controls}).

\vspace{-5pt}

\paragraph{AG-SAE features provide stronger bidirectional causal control.}
The left panel of Figure~\ref{fig:causal_intervention_triptych} shows that
individual AG-SAE features exert substantially stronger native-scale causal
effects in both intervention directions. Mean ablation effect increases from
\(3.23\) to \(3.91\) logits, a \textbf{21.2\%} improvement over Vanilla SAE,
while mean addition effect increases from \(4.22\) to \(4.79\) logits, a
\textbf{13.6\%} improvement. The center panel of
Figure~\ref{fig:causal_intervention_triptych} shows that this advantage is
consistent across concepts. AG-SAE improves addition for \(24/26\) concepts
and ablation for \(21/26\), with \textbf{21/26} improving simultaneously in both
directions. For AG-SAE, L2-matched random directions and unrelated SAE
features produce mean target effects of at most \(0.14\) logits in magnitude
across both intervention directions, establishing the target specificity of
the measured effects. Its leading features therefore function as stronger
causal control units, more effectively writing their associated concepts when
added and erasing them when ablated.

\vspace{-6pt}

\paragraph{Causal control capability is concentrated in the leading feature.}
The right panel of Figure~\ref{fig:causal_intervention_triptych} localizes where
this increased causal capability resides. AG-SAE achieves \textbf{21.0\%} greater
causal yield per active feature at Top-1. The advantage contracts to \(5.1\%\)
at Top-2 and below \(3\%\) thereafter, showing that the gain is strongly
front-loaded into the leading feature. Moreover, for \(20/26\) concepts,
Vanilla SAE requires its two leading selected features to exceed the causal
effect of a single AG-SAE feature. AG-SAE thus exposes target-relevant causal
control through fewer feature units and provides a more compact intervention
interface to the model.

These results reveal an unexpected consequence
of AG-SAE's dictionary--graph self-consistency cycle.
Training to induce faithful feature relations also yields
stronger bidirectional causal control, with gains concentrated in
each concept's leading feature. This finding motivates further research on enhancing SAE features'
causal intervention capabilities beyond reconstruction quality.

\begin{figure}[t]
    \centering
    \includegraphics[width=\linewidth]
    {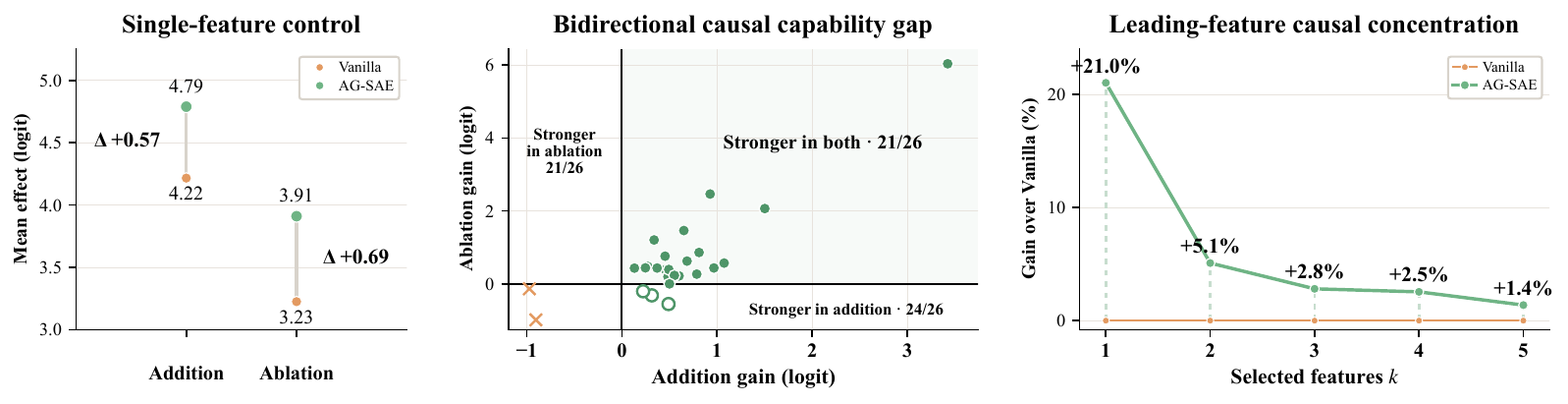}
    \vspace{-15pt}
    \caption{
    \textbf{AG-SAE consolidates bidirectional causal control into its
    leading features.}
    \textbf{Left.} Mean causal effects of adding and ablating the
    leading feature.
    \textbf{Center.} AG-SAE improves both intervention directions
    for \(21/26\) concepts.
    \textbf{Right.} The gain in causal yield per active feature is
    concentrated at Top-1 and narrows as additional features are included.
    Together, the three panels establish stronger magnitude,
    cross-concept consistency, and leading-feature concentration.
    }
    \label{fig:causal_intervention_triptych}
    \vspace{-15pt}
\end{figure}

\vspace{-7pt}

\subsection{Core Component Ablations}
\label{sec:ablation}

\vspace{-5pt}

On the same fixed SAE dictionary and candidate space, replacing
the assembly of independently selected parent--child edges with
Complete Parent-Set Induction raises multi-parent PSV from
\(34.80\%\) to \(66.13\%\) and
\(\mathrm{PSV}\cap\mathrm{NR}\) from \(31.68\%\) to \(64.25\%\),
with identical EV, \(L_0\), and dictionary capacity. Under the SAEBench Feature Absorption protocol~\citep{karvonen2025saebench,chanin2025absorption}, starting from a fixed, fully trained SAE, we perform Complete Parent-Set Induction once
and apply a single Absorption Realignment update using the resulting
graph. Absorption Fraction decreases from 42.08\% to 28.11\% and Full
Absorption from 38.92\% to 27.56\%. Evaluation settings and the inference
threshold remain fixed; EV and mean $L_0$ change by only $-0.000040$ and
$+0.0258$, respectively, with no change in the dead-feature fraction.
These isolated ablations show that complete parent sets improve structural
reliability and that realignment translates recovered structure into
reduced feature absorption. Full protocols and analysis are provided in
Appendix~\ref{app:component_ablations}.

\vspace{-10pt}

\section{Conclusion}
\vspace{-7pt}

AG-SAE extends structured SAE learning from single-parent trees to
mixed topologies allowing features with zero, one, or multiple
parents. It evaluates each candidate parent set as a complete
hypothesis against null, subset, and alternative explanations,
distinguishing jointly necessary multi-parent relations from
spurious or redundant associations. The resulting graph guides SAE training through a structural loss and further guides dictionary refinement through Absorption Realignment and
Residual Completion. Full Re-Induction then reselects every parent set
under the revised dictionary, completing the dictionary--graph
self-consistency cycle. In the controlled toy
model, AG-SAE recovers all \(24\) ground-truth parent sets and
rejects all \(32\) hard negatives. On real LLM activations, it
achieves the strongest single- and multi-parent reliability among
the evaluated methods at comparable representation quality. Its
\(69.56\%\) overall joint validity exceeds the full-space post-hoc
baseline by \(28.93\) percentage points. A concrete case further
illustrates complete parent-set support and a nonredundant child
contribution within a semantically coherent feature relation.
Beyond these structural results, AG-SAE features support stronger
concept writing and concept erasure for \(21\) of \(26\) concepts, with the gains concentrated in each
concept's leading feature. Taken together, AG-SAE not only learns reliable mixed-topology feature structure, but turns that structure from a post-hoc description into an unsupervised training signal that refines the dictionary and yields functional benefits beyond reconstruction quality.

\subsection*{AI use statement}
Generative AI tools were used to assist with literature search and summarization, code editing and debugging, feedback on experimental design and result interpretation, and language and manuscript polishing. The core research idea, methodological novelty, and scientific contributions of this work were conceived by the authors.





\bibliography{references}
\bibliographystyle{iclr2027_conference}

\clearpage
\appendix
\startcontents[appendices]

\begingroup

\setcounter{tocdepth}{2}
\setlength{\parskip}{0pt}


\titlecontents{section}
  [0em]
  {\small\bfseries\addvspace{1.05em}}
  {\contentslabel{1.8em}}
  {}
  {\hfill\contentspage}

\titlecontents{subsection}
  [1.8em]
  {\small\addvspace{0.4em}}
  {\contentslabel{2.7em}}
  {}
  {\titlerule*[0.90em]{.}\contentspage}


{\huge\bfseries Appendix\par}

\vspace{1.55em}

{\Large\bfseries Table of Contents\par}

\vspace{-0.5em}
\noindent\rule{\linewidth}{0.4pt}

\vspace{0.70em}


\noindent
\makebox[\linewidth][c]{%
  \begin{minipage}{0.88\linewidth}
    \printcontents[appendices]{}{1}{}
  \end{minipage}%
}

\vspace{0.7em}
\noindent\rule{\linewidth}{0.4pt}

\endgroup

\vspace{0.85em}

\section{Related Work}
\label{app:related_work}

\paragraph{Sparse Autoencoders and Flat Dictionaries.}
Sparse autoencoders (SAEs) decompose language-model activations into sparse combinations of latent directions drawn from an overcomplete dictionary. Compared with individual neurons, these directions often correspond to more localized and interpretable features \citep{bricken2023monosemanticity,huben2024sparse,marks2025sparse}. Gated and thresholded objectives reduce activation shrinkage and make feature selection more precise. TopK, BatchTopK, and Switch variants further improve the tradeoff between reconstruction and sparsity, distribute sparse activations more effectively, and make dictionary learning easier to scale \citep{rajamanoharan2024gated,rajamanoharan2024jumprelu,gao2025scaling,bussmann2024batchtopk,mudide2025switch,makhzani2014ksparse}. Large-scale projects have consequently released dictionaries with millions of latents across different LLM activation sites \citep{templeton2024scaling,gao2025scaling,lieberum2024gemmascope,he2024llamascope}. Standardized benchmarks evaluate interpretability, disentanglement, feature recovery, and downstream utility in addition to reconstruction and sparsity \citep{karvonen2025saebench,makelov2025principled,karvonen2024measuring,chanin2026synthsaebench}. Beyond describing model representations, SAE features also support causal circuit discovery and targeted interventions \citep{marks2025sparse,braun2024endtoend,ferrando2025know,kharlapenko2025scaling}. Together, these advances have improved how flat feature dictionaries are learned, evaluated, and used.

\paragraph{Feature Pathologies and Hierarchical Organization.}
Sparse reconstruction alone does not determine how semantic variation should be distributed among latents. This ambiguity gives rise to feature splitting, absorption, composition, correlation-induced hedging, and different noncanonical solutions across training runs \citep{bricken2023monosemanticity,chanin2025absorption,bussmann2025matryoshka,dalili2026subspace,jin2026c2r,korznikov2025ortsae,leask2025canonical,cao2026tree,chanin2025hedging,paulo2026different}. Prior work reduces feature instability and redundancy through cross-sample consistency, decoder orthogonalization, and subspace-based representations \citep{jin2026c2r,korznikov2025ortsae,dalili2026subspace}. Other approaches instead focus on organizing SAE features across semantic scales, using post-hoc coactivation analysis or hierarchical sparse coding \citep{ye2024hierarchical,li2025geometry,maas2026conditional,jenatton2011proximal,park2024hierarchicalgeometry}. Matryoshka SAE learns nested dictionary prefixes that move from broad features to finer ones. This organization reduces splitting and absorption and lowers a decoder-similarity proxy for composition \citep{bussmann2025matryoshka}. Matching-Pursuit SAE extracts conditionally orthogonal hierarchical features from successive residuals \citep{costa2025matching}. Together, these approaches organize SAE features at multiple semantic scales, but leave persistent relations between individual features unspecified.

\paragraph{Structured Feature Relations.}
Structured SAEs extend hierarchical feature organization by explicitly modeling relations between individual features. H-SAE links high-level atoms to expert SAEs, HSAE jointly learns multiple levels while maintaining a partial feature forest, and Tree SAE combines parent-gated activation coverage with cumulative reconstruction across levels \citep{muchane2025hierarchical,luo2026hsae,cao2026tree}. Their evaluations cover reconstruction, interpretability, absorption, splitting, composition, and relation reliability \citep{muchane2025hierarchical,luo2026hsae,cao2026tree}. Across the three approaches, explicit coarse-to-fine structure guides representation learning under a single-parent tree or forest prior, which allows each child to have at most one parent. However, post-hoc analyses find both single- and multi-parent associations, showing that feature organization does not uniformly follow a single tree \citep{bussmann2025matryoshka,grandien2026hierarchies}. A post-hoc DAG constructed from pairwise relations removes this topological restriction, but coactivation, redundancy, compositional overlap, and transitive ancestry can introduce spurious relations \citep{grandien2026hierarchies}. To address both the topological restriction and the risk of spurious relations, AG-SAE formulates each complete parent set as the unit of structural inference. It evaluates each candidate set against null, subset, and alternative explanations, thereby learning a mixed topology that contains features with zero, one, or multiple parents. The induced structure is then fed back into dictionary training through
a structural loss and further guides dictionary refinement. The complete
structure is subsequently re-induced from the updated dictionary. Together, this feedback loop couples validated structural inference with representation learning under a mixed topology.

\section{Additional Method Details}
\label{app:additional_method_details}

This section details Complete Parent-Set Induction $\mathcal I$,
the structural training loss $\mathcal L_{\mathcal I}$,
Mixed-Topology-Guided Dictionary Refinement $\mathcal R$,
and the re-induction and joint-stability criteria that close the
AG-SAE dictionary--graph self-consistency cycle.

\subsection{Complete Parent-Set Induction $\mathcal I$}
\label{app:complete_parent_set_induction}

\paragraph{Conditional fitting and evaluation.}
Complete Parent-Set Induction evaluates the current dictionary
$\mathcal F$ using native activations and feature contributions
$z_i(\mathbf x)\mathbf d_i$.
All main-text expectations are estimated by sample averages over their
corresponding conditioning events: representational support and complete-set
comparisons use $a_c=1$, coverage measures the fraction of these samples
with $a_P=1$, and the reconstruction errors defining innovation use
$a_ca_P=1$.

For each fixed $(c,P)$, we fit $\boldsymbol\beta$ separately for
$S=P$ and $S=P\cup\{c\}$ by nonnegative least squares, with one
coefficient per feature contribution shared across samples.
The fitted coefficients are held fixed when evaluating reconstruction
errors on separate comparison samples. Within each data role, both
reconstructions and the zero-reconstruction baseline use the same jointly
active samples. Their errors define $\mathcal S_{\mathrm{inn}}$ in
Eq.~(\ref{eq:child-innovation}), while $\mathcal S_{\mathrm{rep}}$ in
Eq.~(\ref{eq:parent-support}) uses the direct joint parent representation
$\sum_{p\in P}z_p(\mathbf x)\mathbf d_p$.
Candidates require sufficient conditioning-event counts and positive
normalizing energies. Sample requirements and data allocations are given
in Appendix~\ref{app:experimental_setup}, with separate roles for structure
selection, threshold validation, and report-only evaluation.

The population reconstruction errors compare nested function classes on
the same conditional distribution. Since the additional child coefficient
may be zero, whenever the zero-reconstruction error is positive,
\[
0\leq
\mathcal E_{c,P}(P\cup\{c\})
\leq
\mathcal E_{c,P}(P)
\leq
\mathcal E_{c,P}(\varnothing),
\qquad
0\leq\mathcal S_{\mathrm{inn}}(c,P)\leq1.
\]
If the child contribution is already a fixed nonnegative combination of
the parent contributions on this distribution, adding $c$ leaves the
reconstruction class unchanged and the innovation score is zero.

\paragraph{Complete-set comparisons.}
For each nonempty candidate parent set $P$ of child $c$, let $Q$
denote a competing parent set for the same child. We define
\begin{equation}
\label{eq:comparison-set}
\mathcal C_c(P)
=
\{Q:Q\subsetneq P\}
\cup
\{Q:|Q|=|P|,\ Q\neq P\},
\end{equation}
combining all proper subsets of $P$, including the empty set, with
equal-cardinality alternatives. All alternatives are drawn from the
same candidate search space as $P$. Every $Q\in\mathcal C_c(P)$ is evaluated on the same active-child samples as $P$, with each nonempty proper subset receiving its own
support score. For the empty set, the joint parent representation is zero, and we use
\[
a_{\varnothing}=1,
\qquad
\mathcal S_{\mathrm{rep}}(c,\varnothing)=0.
\]
Subset comparisons measure the support gained by retaining the complete
set, while equal-cardinality alternatives test its member identities.
A nonempty candidate is retained only if it satisfies the
representational-support, activation-coverage, and child-innovation
requirements and
\begin{equation}
\label{eq:parent-comparison}
\mathcal S_{\mathrm{rep}}(c,P)
\geq
\max_{Q\in\mathcal C_c(P)}
\mathcal S_{\mathrm{rep}}(c,Q)
+
\tau_{\mathrm{cmp}},
\end{equation}
where $\tau_{\mathrm{cmp}}>0$ is the required complete-set comparison
margin.

\paragraph{Acyclic selection.}
Selection starts from an empty graph with $P_c=\varnothing$ for every
child. At each step, $\mathcal I$ selects, among unassigned children,
the retained feasible candidate $(c,P)$ with the largest
$\mathcal S_{\mathrm{rep}}(c,P)$. A candidate is feasible exactly when
the current graph contains no directed path from $c$ to any $p\in P$,
so adding all edges $p\to c$ preserves acyclicity. Smaller cardinality
breaks score ties, followed by a fixed ordering.

Upon selection, all edges $p\to c$, $p\in P$, are added jointly and
$P_c=P$ is assigned. Selection stops when no feasible retained candidate
remains; unassigned children keep $P_c=\varnothing$. This preserves the
complete parent set as the unit of acceptance and yields the acyclic graph
$G$ used by subsequent structure-guided SAE training and dictionary
refinement.

\subsection{Structural Training Loss $\mathcal L_{\mathcal I}$}
\label{app:structural_training_loss}

The structural loss $\mathcal L_{\mathcal I}$ provides differentiable
training surrogates for the representational-support, complete-set-preference,
activation-coverage, and child-innovation criteria used by Complete
Parent-Set Induction $\mathcal I$. During each fixed-graph training phase,
the induced graph $G$ and comparison sets $\mathcal C_c(P_c)$ remain fixed
while these surrogates guide feature-parameter updates; Full Re-Induction
subsequently reassesses the updated dictionary using the original induction
criteria.

The fixed-graph update uses the same graph-conditioned coordinates as
Absorption Realignment. We retain $(z_i,\mathbf d_i)$ as the canonical
unit-decoder representation and use $(z_i^G,\mathbf d_i^G)$ only as
temporary training coordinates; their construction and conversion back to
the unit-decoder representation are detailed in
Appendix~\ref{app:mixed_topology_dictionary_optimization}.
Define the contribution of feature $i$ as
\[
\mathbf v_i(\mathbf x)
\coloneqq
z_i(\mathbf x)\mathbf d_i
=
z_i^G(\mathbf x)\mathbf d_i^G.
\]

For the multi-bank LLM setting, let $V_b$ denote the features in
reconstruction bank $b$ and $\mathbf b_b$ its reconstruction bias, with
$\mathbf b_b=\mathbf 0$ for a bank without a bias. The graph and complete
parent sets are global over the feature collection, so selected parents may
span banks, whereas reconstruction remains bank-specific:
\[
\widehat{\mathbf x}^{G,b}(\mathbf x)
=
\mathbf b_b
+
\sum_{i\in V_b}
z_i^G(\mathbf x)\mathbf d_i^G.
\]

\paragraph{Sparse activations and gradient evaluation.}
Conditioning events use the native sparse activity indicators and remain
fixed within each gradient update. Let
$\widetilde z_i^G(\mathbf x)\geq0$ denote the nonnegative encoder
response before sparsification in the graph-conditioned coordinates, and
$\widetilde a_i(\mathbf x)\in(0,1)$ its smooth presence gate, obtained
by applying a sigmoid to the temperature-scaled difference between the
encoder preactivation and calibrated activation threshold. The gate
temperature is specified in Appendix~\ref{app:agsae_hyperparameters}.
These smooth gates provide activation-coverage gradients, including for
parents with zero native sparse activation, without replacing the native
activity indicators used by Complete Parent-Set Induction.

\paragraph{Representational support and complete-set preference.}
For a child $c\in V_b$, define the support target
\[
\mathbf y_c(\mathbf x)
\coloneqq
\mathbf x
-
\widehat{\mathbf x}^{G,b}(\mathbf x)
+
\mathbf v_c(\mathbf x).
\]
This is the reconstruction error after removing the child's contribution
from bank $b$. During each gradient update, $\mathbf y_c$ is treated as a
fixed target and gradients flow through the parent prediction.

At the beginning of each fixed-graph phase, nonnegative prediction
coefficients
$\boldsymbol\alpha^{c,P}
=(\alpha_p^{c,P})_{p\in P}$
are fitted separately for each candidate $(c,P)$ on samples with
$a_c(\mathbf x)=1$. Each coefficient is shared across samples and held
fixed throughout the phase. The representational-support surrogate is
\[
\mathcal L_{\mathrm{rep}}(c,P)
=
\frac{
\mathbb E_{\mathbf x\sim\mathcal D}
\!\left[
\left\|
\mathbf y_c(\mathbf x)
-
\sum_{p\in P}
\alpha_p^{c,P}
\widetilde z_p^G(\mathbf x)\mathbf d_p^G
\right\|_2^2
\,\middle|\,
a_c(\mathbf x)=1
\right]
}{
\mathbb E_{\mathbf x\sim\mathcal D}
\!\left[
\left\|
\mathbf y_c(\mathbf x)
\right\|_2^2
\,\middle|\,
a_c(\mathbf x)=1
\right]
}.
\]
The prediction uses the complete set $P$, including parents from
reconstruction banks different from that of the child.

Complete-set preference compares the selected set $P_c$ with every
$Q\in\mathcal C_c(P_c)$ from Eq.~(\ref{eq:comparison-set}) using the
same support target and child-active samples:
\[
\mathcal L_{\mathrm{cmp}}(c)
=
\underset{Q\in\mathcal C_c(P_c)}{\operatorname{mean}}
\;
T\log\!\left(
1+\exp\!\left[
\frac{
m
+
\mathcal L_{\mathrm{rep}}(c,P_c)
-
\mathcal L_{\mathrm{rep}}(c,Q)
}{T}
\right]
\right),
\]
where $m>0$ is the training margin and $T>0$ controls comparison
smoothness. Each competing set uses its own fitted nonnegative prediction
coefficients; for $Q=\varnothing$, the parent prediction is zero.

\paragraph{Activation coverage.}
The coverage criterion used by $\mathcal I$ requires
$\Pr(a_{P_c}=1\mid a_c=1)\geq\tau_{\mathrm{cov}}$.
Its differentiable surrogate encourages every selected parent to be
present on child-active samples:
\[
\mathcal L_{\mathrm{cov}}(c)
=
-
\mathbb E_{\mathbf x\sim\mathcal D}
\!\left[
\sum_{p\in P_c}
\log\widetilde a_p(\mathbf x)
\,\middle|\,
a_c(\mathbf x)=1
\right].
\]
During each gradient update, the child activity indicator is a fixed
conditioning mask, while the smooth parent gates carry encoder gradients.
Full Re-Induction evaluates the original coverage criterion using native
activity indicators.

\paragraph{Child innovation.}
The innovation surrogate prevents the child-specific decoder remainder
from collapsing during graph-conditioned training. Let $\mathbf r_c$
denote this remainder and $\mathbf d_c^G$ its corresponding composite
decoder in the graph-conditioned parameterization of
Appendix~\ref{app:mixed_topology_dictionary_optimization}. We use
\[
\mathcal L_{\mathrm{inn}}(c)
=
\left[
\max\!\left\{
0,\,
\eta
-
\frac{
\|\mathbf r_c\|_2
}{
\|\mathbf d_c^G\|_2
}
\right\}
\right]^2,
\]
where $\eta>0$ is the minimum target for the relative remainder norm.
This discourages the child-specific component from vanishing into its
parent-aligned allocation. It is only a training surrogate:
the reconstruction gain $\mathcal S_{\mathrm{inn}}$ in
Eq.~(\ref{eq:child-innovation}) remains the child-innovation criterion
used by Complete Parent-Set Induction and Full Re-Induction.

\paragraph{Combined structural loss.}
Combining the four components gives
\[
\begin{aligned}
\mathcal L_{\mathcal I}(\mathcal F;G)
=
\lambda_{\mathcal I}
\underset{c:P_c\neq\varnothing}{\operatorname{mean}}
\Bigl[
&\lambda_{\mathrm{rep}}
  \mathcal L_{\mathrm{rep}}(c,P_c)
+
\lambda_{\mathrm{cmp}}
  \mathcal L_{\mathrm{cmp}}(c)
\\
&+
\lambda_{\mathrm{cov}}
  \mathcal L_{\mathrm{cov}}(c)
+
\lambda_{\mathrm{inn}}
  \mathcal L_{\mathrm{inn}}(c)
\Bigr].
\end{aligned}
\]
The nonnegative component weights control the four surrogate contributions,
and the overall structural coefficient $\lambda_{\mathcal I}$ is absorbed
into $\mathcal L_{\mathcal I}$ to match Eq.~(\ref{eq:joint-objective}).
Conditional terms are evaluated only for relations with sufficiently
observed conditioning events and positive normalizing energies
(Appendix~\ref{app:agsae_hyperparameters}); the mean is over evaluable
selected relations, with $\mathcal L_{\mathcal I}=0$ if none are available.

As in Eq.~(\ref{eq:joint-objective}), $\mathcal L_{\mathcal I}$ is combined
with the conventional SAE reconstruction objective. Within each fixed-graph
phase, the selected graph, comparison sets, and fitted prediction
coefficients remain fixed, while conditioning masks are recomputed as the
dictionary evolves. These surrogates guide parameter updates but do not
certify relational validity; Full Re-Induction reassesses the updated
dictionary using the original representational-support, complete-set
comparison, activation-coverage, and child-innovation criteria of
$\mathcal I$.

\subsection{Mixed-Topology-Guided Dictionary Refinement $\mathcal R$}
\label{app:mixed_topology_dictionary_optimization}

The refinement operator $\mathcal R$ uses the induced topology $G$ for
Absorption Realignment and Residual Completion. The structural loss for
the fixed-graph update is specified in
Appendix~\ref{app:structural_training_loss}. We retain the canonical
unit-decoder representation $(z_i,\mathbf d_i)$ and introduce temporary
graph-conditioned coordinates only during this refinement.

\paragraph{Absorption Realignment.}
With $G$ fixed, let $P_i$ denote the selected parent set of feature $i$.
Each feature is represented by a trainable decoder remainder
$\mathbf r_i$ and nonnegative parent-alignment coefficients
$\delta_{ip}$ for $p\in P_i$. Let $\mathbf d_i^G$ denote its composite
decoder direction in graph-conditioned coordinates. Evaluating features in
parent-before-child order gives
\[
\mathbf d_i^G
=
\mathbf r_i
+
\sum_{p\in P_i}
\delta_{ip}\mathbf d_p^G,
\qquad
\delta_{ip}\geq0.
\]
If $P_i=\varnothing$, the sum is empty and
$\mathbf r_i=\mathbf d_i$ at initialization. Otherwise, the
parent-alignment coefficients and $\mathbf r_i$ are initialized by the
nonnegative ridge decomposition in Eq.~(\ref{eq:allocation-fit}). This initialization is repeated at the beginning of every refinement cycle using the current native decoder directions.
Applied in parent-before-child order, this gives
$\mathbf d_i^G=\mathbf d_i$ for every feature at the start of the
fixed-graph update. This graph-conditioned factorization operationalizes the main-text realignment during the fixed-graph update. Because $\mathbf d_i^G$ depends on the shared parent coordinates $\{\mathbf d_p^G \mid p \in P_i\}$, the parent-aligned portion of feature $i$'s contribution is expressed through those coordinates, and its reconstruction gradients propagate through them. The remainder $\mathbf r_i$ parameterizes the child-specific component. The main-text identity gives an equivalent reconstruction-preserving activation view of this computation without requiring an explicit overwrite of native parent activations.

Let $z_i^G(\mathbf x)$ denote the sparse activation of feature $i$ in
graph-conditioned coordinates, initialized as
$z_i^G(\mathbf x)=z_i(\mathbf x)$. For
$\|\mathbf d_i^G\|_2>0$, the corresponding unit-decoder representation is
\begin{equation}
\label{eq:unit-decoder-coordinates}
\mathbf d_i
=
\frac{\mathbf d_i^G}{\|\mathbf d_i^G\|_2},
\qquad
z_i(\mathbf x)
=
\|\mathbf d_i^G\|_2
z_i^G(\mathbf x).
\end{equation}
Thus, the two coordinate systems preserve the same feature contribution,
\[
z_i(\mathbf x)\mathbf d_i
=
z_i^G(\mathbf x)\mathbf d_i^G,
\]
and, because the rescaling factor is positive, the activity indicator
$a_i(\mathbf x)$ is unchanged.

Using the reconstruction-bank notation introduced in
Appendix~\ref{app:structural_training_loss}, let $V_b$ denote the
features in bank $b$ and $\mathbf b_b$ its reconstruction bias. The
graph-conditioned reconstruction of bank $b$ is
\[
\widehat{\mathbf x}^{G,b}(\mathbf x)
=
\mathbf b_b
+
\sum_{i\in V_b}
z_i^G(\mathbf x)\mathbf d_i^G.
\]
The graph and complete parent sets are global, so a selected parent of a
feature in $V_b$ may belong to a different reconstruction bank, while
reconstruction remains bank-specific. Since
$z_i^G=z_i$ and $\mathbf d_i^G=\mathbf d_i$ at initialization, the
graph-conditioned parameterization preserves the current reconstruction
of every bank.

During the fixed-graph update, encoder parameters, decoder remainders,
and parent-alignment coefficients are optimized in graph-conditioned
coordinates under the reconstruction loss and the structural loss of
Appendix~\ref{app:structural_training_loss}, as summarized in
Eq.~(\ref{eq:joint-objective}). The main-text notation suppresses the
temporary graph-conditioned coordinates; $\mathcal F$ denotes the
corresponding unit-decoder feature representation.
After the fixed-graph update, the temporary coordinates are converted
back to the unit-decoder representation using
Eq.~(\ref{eq:unit-decoder-coordinates}), preserving each feature
contribution. The resulting native sparse activations and unit-norm
decoder directions are then used by Residual Completion.

\paragraph{Residual Completion.}
For a child $c\in V_b$ with selected parent set $P_c$, Residual
Completion uses the native reconstruction of the child's bank $b$ with
graph routing disabled:
\[
\widehat{\mathbf x}^{\,b}(\mathbf x)
=
\mathbf b_b
+
\sum_{i\in V_b}
z_i(\mathbf x)\mathbf d_i,
\qquad
\mathbf e_b(\mathbf x)
=
\mathbf x-\widehat{\mathbf x}^{\,b}(\mathbf x).
\]
Here, $\mathbf e_b(\mathbf x)$ is the native reconstruction residual of
bank $b$. The parent--child-conditioned residual is
\[
\mathbf e_c(\mathbf x)
=
a_c(\mathbf x)
a_{P_c}(\mathbf x)
\mathbf e_b(\mathbf x).
\]
Thus, the residual comes from the child's reconstruction bank, while the
conditioning event requires the child and all selected parents to be
active, including parents from other banks.

Leading directions of the uncentered residual second moment
\[
\mathbb E_{\mathbf x\sim\mathcal D}
\!\left[
\mathbf e_c(\mathbf x)
\mathbf e_c(\mathbf x)^\top
\right]
\]
provide candidates for the completion objective in
Eq.~(\ref{eq:residual-completion}), oriented by the reconstruction
energy captured under nonnegative projection. For a fixed unit direction
$\mathbf d$, let $u\geq0$ denote its nonnegative amplitude on a residual
sample. The optimal amplitude for the sparse rank-one completion objective
is
\[
\underset{u\geq0}{\arg\min}
\left\{
\|\mathbf e_c(\mathbf x)-u\mathbf d\|_2^2
+
\lambda_{\mathrm{sp}}u
\right\}
=
\max\!\left\{
0,\,
\mathbf d^\top\mathbf e_c(\mathbf x)
-
\frac{\lambda_{\mathrm{sp}}}{2}
\right\},
\]
where $\lambda_{\mathrm{sp}}>0$ is the sparsity coefficient from
Eq.~(\ref{eq:residual-completion}). This gives the nonnegative activation
target for a fixed residual direction; candidate validation and encoder
initialization construct a learned approximation to it.

A candidate residual component must have sufficient contextual
observations and retain reconstruction gain on separate validation
samples. The observation, gain, and fit-to-validation consistency
thresholds are specified in
Appendix~\ref{app:agsae_hyperparameters}.

Residual Completion first tests whether an existing feature in the same
reconstruction bank $V_b$ can capture the accepted residual direction.
A directionally matched feature is reused only if its native activation
and decoder contribution also satisfy the conditional reconstruction-gain
criterion on validation samples. An accepted reused feature retains its
existing identity and participates in subsequent dictionary refinement.
If no existing feature in $V_b$ passes the reuse criterion, an unused
entry in bank $b$ is initialized with the accepted residual direction and
an activation threshold calibrated to the prescribed sparsity. The
initialized entry receives a new feature identity, becomes part of
$\mathcal F^{(t+1)}$, and has no assigned parent set under the current
fixed graph $G$. Its encoder is further optimized in subsequent
fixed-graph phases using Eq.~(\ref{eq:joint-objective}). Its parent set is
not inherited from the relation that produced the residual; instead, Full
Re-Induction determines its complete parent set together with those of all
other features.

\subsection{Full Re-Induction and Joint Stability}
\label{app:full_reinduction_stability}

\paragraph{Full Re-Induction.}
After Mixed-Topology-Guided Dictionary Refinement $\mathcal R$ produces
$\mathcal F^{(t+1)}$, the temporary graph-conditioned coordinates are
converted back through Eq.~(\ref{eq:unit-decoder-coordinates}) to the
native unit-decoder representation $(z_i,\mathbf d_i)$, preserving each
feature contribution.

Full Re-Induction then reapplies $\mathcal I$ directly to
$\mathcal F^{(t+1)}$ using the updated native sparse activations
$z_i(\mathbf x)$, activity indicators $a_i(\mathbf x)$ recomputed from
them, and feature contributions $z_i(\mathbf x)\mathbf d_i$. Every
feature, including newly initialized ones, is reassessed under the
candidate search rules of
Appendix~\ref{app:complete_parent_set_induction}; candidate coefficients
and scores are recomputed from the updated feature contributions,
followed by complete-set comparisons and acyclic selection:
\[
G^{(t+1)}
=
\mathcal I\!\left(\mathcal F^{(t+1)}\right)
=
\left(P_c^{(t+1)}\right)_{c=1}^{M_{t+1}}.
\]
Existing parent sets compete under the same criteria as all other
candidates, so each feature's assignment may change with its updated
representation.

\paragraph{Feature identity and stability.}
Persistent features retain their identities across refinement cycles;
reuse preserves an existing identity, while each newly initialized
feature receives a new one. Let $\mathcal V^{(t)}$ denote the feature
identities in $\mathcal F^{(t)}$, with
$|\mathcal V^{(t)}|=M_t$, and define
\[
\mathcal U
\coloneqq
\mathcal V^{(t)}
\cup
\mathcal V^{(t+1)}
\]
as the identities present in either successive dictionary state.
Parent-set members are compared across cycles using these persistent
identities.

Structural change is the fraction of identities whose membership or
complete parent-set assignment changes:
\[
d_G\!\left(G^{(t+1)},G^{(t)}\right)
=
1-
\frac{
\left|
\left\{
c\in
\mathcal V^{(t)}\cap\mathcal V^{(t+1)}
:
P_c^{(t+1)}=P_c^{(t)}
\right\}
\right|
}{
|\mathcal U|
}.
\]
The numerator counts persistent features with unchanged complete parent
sets; newly introduced features therefore contribute to structural
change even when assigned an empty parent set. Hence $d_G=0$ exactly
when feature membership and all parent-set assignments are unchanged.

Feature change compares individual feature contributions on the same
fixed validation samples. For any vector-valued function
$\boldsymbol\phi(\mathbf x)$, define
\[
\|\boldsymbol\phi\|_{\mathrm{val}}^2
\coloneqq
\mathbb E_{\mathbf x\sim\mathcal D_{\mathrm{val}}}
\!\left[
\|\boldsymbol\phi(\mathbf x)\|_2^2
\right],
\]
where $\mathcal D_{\mathrm{val}}$ is the validation distribution used
for stability evaluation and the expectation is estimated by sample
average. The normalized feature change is
\[
d_F\!\left(
\mathcal F^{(t+1)},
\mathcal F^{(t)}
\right)
=
\left(
\frac{
\displaystyle
\sum_{i\in\mathcal U}
\left\|
z_i^{(t+1)}(\cdot)\mathbf d_i^{(t+1)}
-
z_i^{(t)}(\cdot)\mathbf d_i^{(t)}
\right\|_{\mathrm{val}}^2
}{
\displaystyle
\sum_{i\in\mathcal U}
\left(
\left\|
z_i^{(t+1)}(\cdot)\mathbf d_i^{(t+1)}
\right\|_{\mathrm{val}}^2
+
\left\|
z_i^{(t)}(\cdot)\mathbf d_i^{(t)}
\right\|_{\mathrm{val}}^2
\right)
}
\right)^{1/2}.
\]
For an identity absent from one dictionary state, its contribution is
defined as zero in that state; we also set $d_F=0$ when the denominator
is zero. Contributions are compared per feature before aggregation,
preventing cancellation across feature updates. Since $d_F$ depends on
$z_i(\mathbf x)\mathbf d_i$, it is invariant to
contribution-preserving activation--decoder rescaling.

\paragraph{Stopping and returned state.}
Joint stability is evaluated after both dictionary refinement and
Full Re-Induction have completed for the current cycle. For
$\gamma_G,\gamma_F>0$, AG-SAE terminates when
\[
d_G\!\left(G^{(t+1)},G^{(t)}\right)
\leq
\gamma_G,
\qquad
d_F\!\left(
\mathcal F^{(t+1)},
\mathcal F^{(t)}
\right)
\leq
\gamma_F.
\]
When both conditions hold, AG-SAE returns
\[
(\mathcal F^\star,G^\star)
=
(\mathcal F^{(t+1)},G^{(t+1)}),
\qquad
G^\star
=
\mathcal I(\mathcal F^\star).
\]
Thus, the returned graph is induced from the returned dictionary, and
the completed cycle satisfies both structural and feature-contribution
stability tolerances. Stopping uses validation samples separate from report-only evaluation;
their allocation and the tolerances are specified in
Appendix~\ref{app:experimental_setup}, while
Appendix~\ref{app:stability_theory} provides sufficient conditions for eventual graph stability and joint stopping.

\section{Theoretical Analysis of Graph Stability in Full Re-Induction}
\label{app:stability_theory}

We analyze the stability of the graph produced by Full Re-Induction as the dictionary evolves across refinement cycles. Small dictionary changes could in principle alter the discrete graph through changes in candidate acceptance and selection. We show that around a nondegenerate stabilized state, sufficiently small changes preserve all complete parent sets, preventing repeated structural changes from blocking the joint stopping rule.

\subsection{Setup and Assumptions}

Let $\mathcal V$ be a finite, nonempty set of feature identities, with
$M=|\mathcal V|$, and let $\theta\in\Theta\subseteq\mathbb R^q$
denote the continuous state determining the native activations, decoder
directions, and induction thresholds. Write
\[
\mathcal F_\theta
=\{(z_i(\cdot;\theta),\mathbf d_i(\theta))\}_{i\in\mathcal V},
\qquad
\mathbf v_i(\mathbf x;\theta)
=z_i(\mathbf x;\theta)\mathbf d_i(\theta),
\]
where $z_i\geq0$ and $\|\mathbf d_i\|_2=1$.
For multiple banks, $\mathcal V$ contains the features from all banks,
and parent sets may cross banks. Reconstruction follows the bankwise
objectives in Appendices~\ref{app:structural_training_loss}
and~\ref{app:mixed_topology_dictionary_optimization}.
Each refinement step returns to this native representation through the
contribution-preserving conversion in
Eq.~(\ref{eq:unit-decoder-coordinates}).

Fix a finite empirical induction configuration $\Xi$, comprising the
retrieval, fitting, independent comparison, and control-calibration data,
batch contexts, external randomness, and deterministic tie rules.
For BatchTopK, occurrences of an input in different batch contexts are
distinct sample positions. We write
$\mathcal I_\Xi(\mathcal F_\theta)$ for the induction operator with
this configuration made explicit. It recomputes sparse supports,
candidate families, fitted coefficients, calibrated thresholds, and the
graph, including dictionary-dependent control retrieval.

The empirical scores are those in
Appendix~\ref{app:complete_parent_set_induction}.
The support score $\mathcal S_{\mathrm{rep}}$ compares
$\mathbf v_c$ with the direct sum $\sum_{p\in P}\mathbf v_p$.
The innovation score $\mathcal S_{\mathrm{inn}}$ uses separate
nonnegative least-squares (NNLS) fits for $P$ and $P\cup\{c\}$,
evaluated on independent comparison samples.
Let $\operatorname{cov}(c,P;\theta)$ denote the fraction of active-child
samples on which all parents in $P$ are active. For each nonempty
candidate passing the sample-count and energy requirements, define its
acceptance margin by
\begin{equation}
\label{eq:agsae-acceptance-margin}
\begin{aligned}
m_{c,P}(\theta)=\min\bigl\{\,&
\operatorname{cov}(c,P;\theta)-\tau_{\mathrm{cov}},\\
&\mathcal S_{\mathrm{rep}}(c,P;\theta)
-\tau_{\mathrm{rep}}(c,P;\theta),\\
&\mathcal S_{\mathrm{inn}}(c,P;\theta)-\tau_{\mathrm{inn}},\\
&\mathcal S_{\mathrm{rep}}(c,P;\theta)
-\max_{Q\in\mathcal C_c(P)}\mathcal S_{\mathrm{rep}}(c,Q;\theta)
-\tau_{\mathrm{cmp}}\bigr\}.
\end{aligned}
\end{equation}
Here $\mathcal C_c(P)$ is the comparison family in
Eq.~(\ref{eq:comparison-set}), including the empty set with
$\mathcal S_{\mathrm{rep}}(c,\varnothing)=0$.
The threshold $\tau_{\mathrm{rep}}(c,P;\theta)$ is calibrated as in
Appendix~\ref{app:agsae_hyperparameters}; its arguments record the
dictionary dependence and calibration group of $(c,P)$.
Candidates in the same model, dictionary-width, and parent-set-cardinality
group share the threshold.
A candidate is retained exactly when $m_{c,P}(\theta)\geq0$.
Starting from the empty graph, acyclic selection repeatedly chooses the
feasible retained candidate with the largest support score among
unassigned children and adds its complete parent set.

\begin{agsaecondition}[Nondegenerate reference state]
\label{cond:agsae-nondegeneracy}
Let $\bar\theta\in\Theta$ be a valid native state. The reference
execution of $\mathcal I_\Xi$ at $\bar\theta$ satisfies:
\begin{enumerate}
\renewcommand{\labelenumi}{(\roman{enumi})}
\item \textbf{Discrete comparisons and branch continuity.}
The comparisons determining native supports, eligible samples,
preliminary screening, candidate pools, search spaces, and control
identities form a finite collection. After substituting earlier reference
choices, each comparison is represented by a scalar function
$h_j(\theta)$ continuous in a relative neighborhood of $\bar\theta$.
Either $h_j(\bar\theta)\neq0$, or $h_j$ is identically zero in that
neighborhood and uses a fixed tie rule. Activation amplitudes, decoder
directions, and preprocessing quantities are continuous on the reference
branch. Sparse supports are described by the comparisons generating them,
allowing inactive outputs to remain identically zero.

\item \textbf{Well-posed fitting and calibration.}
Every nonempty NNLS problem solved on the reference data branch has a
fitting matrix of full column rank. All normalizing energies used in
scoring and calibration are positive, and each calibration group has a
finite, nonempty list of valid control scores.

\item \textbf{Strict acceptance margins.}
Every candidate passing preliminary screening has
$m_{c,P}(\bar\theta)\neq0$.

\item \textbf{Strict selection gaps.}
At each step of acyclic selection, the selected candidate has a strictly
larger support score than every other feasible retained candidate.
This condition is vacuous when only one candidate is feasible.
\end{enumerate}
\end{agsaecondition}

\subsection{Local Stability of Full Re-Induction}

To show that small dictionary changes preserve the induced graph,
we first control how refitting changes the relation scores.
The following lemma establishes that the NNLS coefficients vary
continuously with their inputs, so refitting does not introduce
abrupt changes in the reconstruction errors used to assess
child innovation. This sensitivity property follows from the
theory of solution mappings for strongly monotone variational
inequalities \citep[Section~2F]{agsae_dontchev2014implicit}.
We give a direct argument with an explicit perturbation bound.

\begin{agsaelemma}[Continuity of NNLS solutions]
\label{lem:agsae-nnls-continuity}
Suppose $A(\theta)$ and $\mathbf y(\theta)$ are continuous in a
neighborhood of $\bar\theta$, and $A(\bar\theta)$ has full column
rank. Then, in a sufficiently small neighborhood, the solution
\[
\boldsymbol\beta(\theta)
=\operatorname*{arg\,min}_{\boldsymbol\beta\geq\mathbf0}
\tfrac12\|A(\theta)\boldsymbol\beta-\mathbf y(\theta)\|_2^2
\]
exists, is unique, and is continuous in $\theta$.
\end{agsaelemma}

\begin{proof}
Write $H=A^\top A$ and $\mathbf w=A^\top\mathbf y$.
Full column rank and continuity give $H\succeq\mu I$ for some
$\mu>0$ throughout a sufficiently small neighborhood. The objective
is strongly convex and coercive on the closed nonnegative orthant, so
its minimizer exists and is unique.

For two nearby instances, distinguish their data and solutions by primes.
The first-order optimality conditions
\citep[Section~4.2.3]{agsae_boyd2004convex} give
\[
\begin{aligned}
\langle H\boldsymbol\beta-\mathbf w,
\boldsymbol\beta'-\boldsymbol\beta\rangle&\geq0,\\
\langle H'\boldsymbol\beta'-\mathbf w',
\boldsymbol\beta-\boldsymbol\beta'\rangle&\geq0.
\end{aligned}
\]
Adding the inequalities and setting
$\boldsymbol\Delta=\boldsymbol\beta'-\boldsymbol\beta$ yields
\[
\mu\|\boldsymbol\Delta\|_2^2
\leq
\langle\boldsymbol\Delta,
\mathbf w'-\mathbf w-(H'-H)\boldsymbol\beta\rangle.
\]
By the Cauchy--Schwarz inequality,
\begin{equation}
\label{eq:agsae-nnls-perturbation}
\|\boldsymbol\beta'-\boldsymbol\beta\|_2
\leq
\frac{\|\mathbf w'-\mathbf w\|_2
+\|H'-H\|_{\mathrm{op}}\|\boldsymbol\beta\|_2}{\mu}.
\end{equation}
The bound also holds when $\boldsymbol\Delta=\mathbf0$.
Fixing the unprimed instance and letting the primed instance approach it
proves continuity.
\end{proof}

We now show how continuous score changes lead to unchanged
parent-set decisions. Under the nondegeneracy conditions,
sufficiently small dictionary perturbations preserve the
comparisons governing candidate acceptance and acyclic
selection. Full Re-Induction therefore recovers the same
mixed topology throughout a neighborhood of the reference state.

\begin{agsaeproposition}[Local constancy of Full Re-Induction]
\label{prop:agsae-local-stability}
If Condition~\ref{cond:agsae-nondegeneracy} holds at $\bar\theta$,
there exists $r_{\mathcal I}>0$ such that
\begin{equation}
\label{eq:agsae-local-constancy}
\theta\in\Theta,\quad\|\theta-\bar\theta\|_2<r_{\mathcal I}
\quad\Longrightarrow\quad
\mathcal I_\Xi(\mathcal F_\theta)
=\mathcal I_\Xi(\mathcal F_{\bar\theta}).
\end{equation}
Thus every complete parent set and the global DAG remain unchanged,
including features with zero, one, or multiple parents.
\end{agsaeproposition}

\begin{proof}
We preserve the discrete execution branch, establish continuity of the
scores and thresholds, and then use the strict margins to preserve
candidate acceptance and acyclic selection.

\smallskip
\noindent\textbf{Step 1: Discrete execution branch.}
Consider the comparisons in
Condition~\ref{cond:agsae-nondegeneracy}(i).
For each $h_j(\bar\theta)\neq0$, continuity gives a neighborhood on
which
\[
|h_j(\theta)-h_j(\bar\theta)|
<\tfrac12|h_j(\bar\theta)|.
\]
Intersect these finitely many neighborhoods with those preserving the
identically zero comparisons. If an execution first departed from the
reference choices, all preceding choices would agree, so the departing
comparison would use the same reference expression $h_j$.
Its sign or fixed tie decision is unchanged, a contradiction.
Hence the supports, conditional samples, preliminary screening,
candidate families, and control identities agree with the reference
execution. Sample counts and coverage fractions are therefore constant.

\smallskip
\noindent\textbf{Step 2: Refitted scores.}
On this data branch, the feature contributions are continuous, so the
NNLS matrices and targets are continuous.
Condition~\ref{cond:agsae-nondegeneracy}(ii) and
Lemma~\ref{lem:agsae-nnls-continuity} give continuous fitted
coefficients. For each refitted reconstruction, let
$n_{\mathrm{cmp}}>0$ be the number of comparison samples.
Let $A_{\mathrm{cmp}}(\theta)$ and
$\mathbf y_{\mathrm{cmp}}(\theta)$ denote the corresponding
design matrix and target vector, respectively, obtained by
stacking feature contributions and reconstruction targets
across these samples. The coefficient vector
$\boldsymbol\beta(\theta)$ is estimated on the separate fitting
samples. Each error evaluated on the fixed comparison samples
has the continuous form
\[
\frac{1}{n_{\mathrm{cmp}}}
\|A_{\mathrm{cmp}}(\theta)\boldsymbol\beta(\theta)
-\mathbf y_{\mathrm{cmp}}(\theta)\|_2^2,
\qquad n_{\mathrm{cmp}}>0.
\]
The innovation score is a difference of these errors divided by a
positive energy, and is therefore continuous. The direct support score
is likewise continuous, since it consists of finite averages of
continuous contributions with a positive denominator.

\smallskip
\noindent\textbf{Step 3: Calibration and acceptance.}
Let $s_1(\theta),\ldots,s_K(\theta)$ be the continuous scores of a
fixed control list. Their $k$th order statistic satisfies
\[
|s_{(k)}(\theta)-s_{(k)}(\bar\theta)|
\leq\max_{1\leq j\leq K}|s_j(\theta)-s_j(\bar\theta)|.
\]
Indeed, perturbing every entry by at most $\delta$ changes each ordered
entry by at most $\delta$. Standard sample quantiles select or
interpolate order statistics \citep{agsae_hyndman1996quantiles}.
For a fixed list length, quantile level, and rule, the interpolation
weights are fixed, so the quantile is continuous. Adding the calibration
margin and taking the maximum with the floor preserve continuity;
hence $\tau_{\mathrm{rep}}$ is continuous.
Finite maxima and minima then give continuity of the acceptance margin
in Eq.~(\ref{eq:agsae-acceptance-margin}).
By Condition~\ref{cond:agsae-nondegeneracy}(iii), a smaller
neighborhood preserves the signs of all acceptance margins and thus
the retained candidate set.

\smallskip
\noindent\textbf{Step 4: Acyclic selection.}
Both executions start from the empty graph. If their first $k$ choices
agree, their current graphs and unassigned children agree.
Reachability from each child to its candidate parents is identical,
so the feasible candidates are the same. The reference winner's score
gap over each competitor is continuous and strictly positive by
Condition~\ref{cond:agsae-nondegeneracy}(iv). It therefore remains
the winner in a smaller neighborhood.
Induction over at most $M$ selections gives the same complete parent
sets and the same termination point; all remaining children keep empty
parent sets. Only finitely many neighborhood restrictions are required,
yielding a common positive radius $r_{\mathcal I}$.
\end{proof}

\paragraph{Decision margins and the stability radius.}
Let $\mathcal H$ collect the nonzero reference comparisons, acceptance
margins, and selection score gaps whose signs are preserved in the
proof. Suppose these functions have Lipschitz constants $L_h$ on a
relative ball of radius $r_0>0$ that also preserves full rank, positive
energies, and persistent ties. Then one may take
\begin{equation}
\label{eq:agsae-stability-radius}
r_{\mathcal I}
=\min\left\{r_0,\;
\min_{h\in\mathcal H:\,L_h>0}
\frac{|h(\bar\theta)|}{2L_h}\right\},
\end{equation}
where the inner minimum is $+\infty$ when its index set is empty.
The admissible perturbation is thus controlled by the decision margins
relative to their sensitivity to the dictionary state.

\subsection{Eventual Graph Stability and Joint Stopping}

Local stability provides the graph-side guarantee needed for joint stopping.
We next consider a refinement sequence whose dictionary state stabilizes to
a nondegenerate limit. Once the dictionary states remain within the
corresponding stability neighborhood, Full Re-Induction returns the same
graph. Under the continuity conditions below, feature-contribution changes
also vanish, ensuring that the joint stopping rule in
Appendix~\ref{app:full_reinduction_stability} is eventually satisfied.
Fix nonempty validation samples $X_{\mathrm{val}}$ with fixed batch
contexts, separate from report-only evaluation. The sample-average
validation norm is
\[
\|\mathbf u\|_{\mathrm{val}}^2
\coloneqq\frac{1}{|X_{\mathrm{val}}|}
\sum_{\mathbf x\in X_{\mathrm{val}}}\|\mathbf u(\mathbf x)\|_2^2.
\]
Empirically, the refinement trajectories in
Figure~\ref{fig:graph_refinement_dynamics} exhibit a stabilization trend
consistent with the regime considered below. After $G^{(2)}$, both $d_F$
and $d_G$ decrease monotonically.

\begin{agsaecorollary}[Eventual stability and finite joint stopping]
\label{cor:agsae-joint-stopping}
Consider a continued sequence of the stated refinement updates, with
$\mathcal F^{(t)}=\mathcal F_{\theta_t}$ and
$G^{(t)}=\mathcal I_\Xi(\mathcal F_{\theta_t})$ under the same
configuration $\Xi$ at every cycle. Suppose that:
\begin{enumerate}
\renewcommand{\labelenumi}{(\roman{enumi})}
\item After some finite $t_0$, the feature identity set is the same
finite, nonempty $\mathcal V$. New identities may be introduced before
$t_0$.
\item The states satisfy $\theta_t\to\theta_\infty\in\Theta$, and
Condition~\ref{cond:agsae-nondegeneracy} holds at $\theta_\infty$.
\item Every validation contribution $\mathbf v_i(\mathbf x;\theta)$
is continuous at $\theta_\infty$, and
\[
E_{\mathrm{val}}^\infty
\coloneqq\sum_{i\in\mathcal V}
\|\mathbf v_i(\cdot;\theta_\infty)\|_{\mathrm{val}}^2>0.
\]
\end{enumerate}
Let $G^\infty=\mathcal I_\Xi(\mathcal F_{\theta_\infty})$.
There exists a finite $T_G\geq t_0$ such that
\[
\begin{aligned}
G^{(t)}&=G^\infty,\qquad
d_G(G^{(t+1)},G^{(t)})=0\quad(t\geq T_G),\\
d_F(\mathcal F_{\theta_{t+1}},\mathcal F_{\theta_t})&\longrightarrow0.
\end{aligned}
\]
Consequently, for any fixed $\gamma_G,\gamma_F>0$, there exists a
finite $T_\gamma$ such that both joint-stability inequalities hold
for every $t\geq T_\gamma$.
\end{agsaecorollary}

\begin{proof}
Apply Proposition~\ref{prop:agsae-local-stability} at
$\theta_\infty$. Since $\theta_t\to\theta_\infty$, all sufficiently
late states lie in its stability neighborhood. Thus
$G^{(t)}=G^\infty$ for every $t\geq T_G$, for some finite
$T_G\geq t_0$.

Let $W(\theta)$ concatenate all validation contributions in a fixed
feature and sample order, scaled by $|X_{\mathrm{val}}|^{-1/2}$.
On the fixed-identity tail, the metrics in
Appendix~\ref{app:full_reinduction_stability} are exactly
\begin{equation}
\label{eq:agsae-tail-metrics}
\begin{aligned}
d_G(G^{(t+1)},G^{(t)})
&=\frac1M\sum_{c\in\mathcal V}
\mathbf1\{P_c^{(t+1)}\neq P_c^{(t)}\},\\
d_F(\mathcal F_{\theta_{t+1}},\mathcal F_{\theta_t})
&=\frac{\|W(\theta_{t+1})-W(\theta_t)\|_2}
{\bigl(\|W(\theta_{t+1})\|_2^2+\|W(\theta_t)\|_2^2\bigr)^{1/2}}.
\end{aligned}
\end{equation}
As in the original definition, $d_F=0$ when its denominator is zero.
Eventual equality of the graphs gives $d_G=0$.
Continuity of the finitely many validation contributions gives
$W(\theta_t)\to W(\theta_\infty)$, and hence
\[
\begin{aligned}
\|W(\theta_{t+1})-W(\theta_t)\|_2
&\leq\|W(\theta_{t+1})-W(\theta_\infty)\|_2\\
&\quad+\|W(\theta_t)-W(\theta_\infty)\|_2
\longrightarrow0.
\end{aligned}
\]
The squared denominator in Eq.~(\ref{eq:agsae-tail-metrics}) tends
to $2E_{\mathrm{val}}^\infty>0$, proving $d_F\to0$.
Choose $T_F$ such that $d_F\leq\gamma_F$ for all $t\geq T_F$, and set
$T_\gamma=\max\{T_G,T_F\}$. Both stopping inequalities then hold
for every $t\geq T_\gamma$. Checking them after each completed cycle
therefore triggers joint stopping no later than cycle $T_\gamma$
along this continuation.
\end{proof}

Corollary~\ref{cor:agsae-joint-stopping} gives sufficient conditions
for reaching a stopping cycle. At any stopping cycle $t_s$, the returned
dictionary--graph pair satisfies
\[
(\mathcal F^\star,G^\star)
=(\mathcal F_{\theta_{t_s+1}},G^{(t_s+1)}),
\qquad
G^\star=\mathcal I_\Xi(\mathcal F^\star).
\]
Thus the returned graph is induced from the returned dictionary.
Positive tolerances allow the joint stopping criterion to be met before
the graph becomes exactly constant.

Together, these theoretical results show that, under the stated conditions,
once the dictionary stabilizes around a nondegenerate state, Full Re-Induction
eventually preserves all complete parent sets. The graph therefore also
stabilizes rather than repeatedly changing structure, allowing the joint
stopping criterion to be satisfied.

\section{Additional Experimental Results}
\label{app:additional_results}

\subsection{Additional Analysis of the Refinement Process and Learned Mixed Topology}
\label{app:graph_coverage_validity}

On Gemma-2-2B layer-13 activations from MiniPile, we analyze
how the dictionary--graph pair evolves during refinement and
how the resulting feature relations are organized across
dictionary scales and parent-set arities.

\begin{figure}[H]
    \centering
    \includegraphics[width=0.315\linewidth]
        {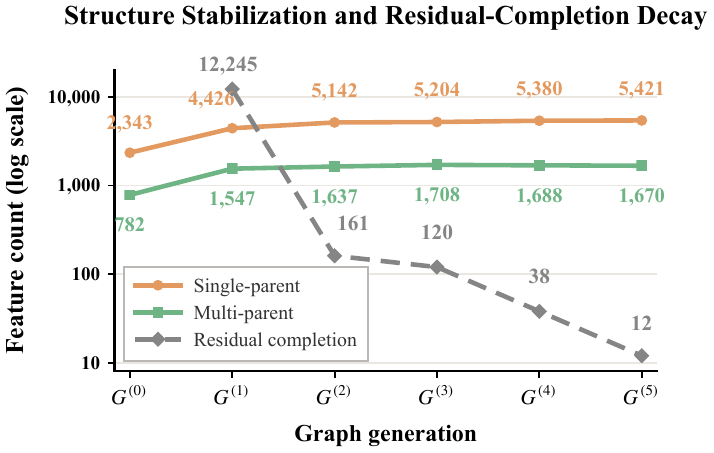}
    \hfill
    \includegraphics[width=0.315\linewidth]
        {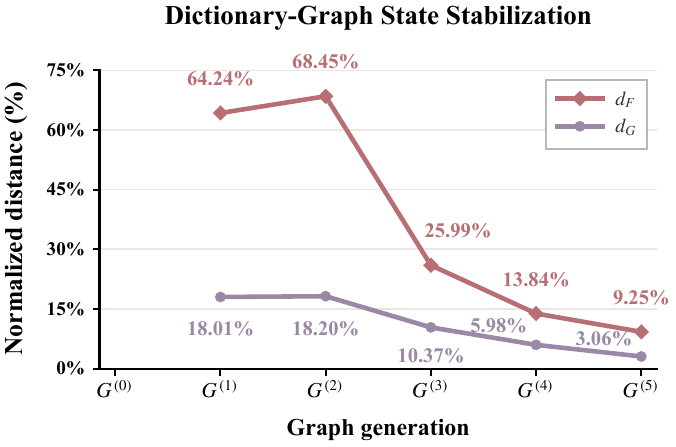}
    \hfill
    \includegraphics[width=0.315\linewidth]
        {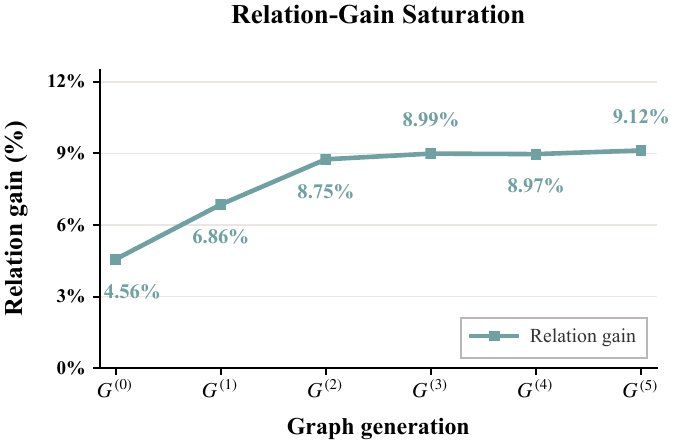}
    \caption{
    \textbf{AG-SAE refinement dynamics and joint stability.}
    Results use Gemma-2-2B layer-13 activations from MiniPile.
    \textbf{Left.} Counts of single- and multi-parent features across
    graph generations, together with accepted Residual Completion
    updates between successive refinement cycles.
    \textbf{Middle.} Normalized feature-contribution change \(d_F\)
    on held-out activations and graph change \(d_G\) in feature
    membership and complete parent-set assignments between successive
    refinement cycles. Both decrease after \(G^{(2)}\), reaching
    \(9.25\%\) and \(3.06\%\), respectively, in the final transition.
    \textbf{Right.} Relation gain is the held-out fraction of residual representation explained by parents with the child withheld, averaged over all selected relations, increasing from \(4.56\%\) at \(G^{(0)}\) to \(9.12\%\) at \(G^{(5)}\).
    }
    \label{fig:graph_refinement_dynamics}
\end{figure}

\paragraph{Refinement Across Cycles.}
The state snapshots in Panel C of
Table~\ref{tab:ag-sae-combined} summarize the refinement statistics at
representative graph generations, with Absorption Realignment counting
the relations undergoing realignment during the corresponding refinement
cycle. The complete trajectories in
Figure~\ref{fig:graph_refinement_dynamics} show that AG-SAE
first forms most of its mixed topology and then progressively
consolidates the dictionary--graph pair. By \(G^{(2)}\), the
single- and multi-parent populations are already close to their
final levels. The multi-parent population peaks at \(1{,}708\)
in \(G^{(3)}\) and then declines slightly to \(1{,}670\) by
\(G^{(5)}\), showing that AG-SAE consolidates its learned topology
rather than relying on unrestricted structural expansion. The left
panel of Figure~\ref{fig:graph_refinement_dynamics} shows that
accepted Residual Completion updates fall from \(12{,}245\) in the
first transition to \(12\) in the final transition. This decline
indicates that progressively fewer completion updates are accepted
as refinement proceeds, consistent with Residual Completion
addressing representational gaps exposed by the induced structure
and the remaining demand for such corrections diminishing. The
middle panel shows that feature-contribution change \(d_F\)
increases from \(64.24\%\) to \(68.45\%\) across the first two
transitions, coinciding with the period of greatest topology growth
and Residual Completion activity. After \(G^{(2)}\), \(d_F\)
decreases monotonically to \(9.25\%\). This non-monotonic trajectory is consistent with an early dictionary-reorganization phase followed by the later consolidation
of feature contributions. The same panel shows that
the exact parent-set change fraction \(d_G\) falls from approximately
\(18\%\) in the first two transitions to \(3.06\%\) in the last,
indicating that Full Re-Induction increasingly retains the parent
sets selected under the preceding dictionary. Meanwhile, the right
panel shows that mean held-out relation gain doubles from
\(4.56\%\) at \(G^{(0)}\) to \(9.12\%\) at \(G^{(5)}\), remaining
near \(9\%\) after \(G^{(2)}\). Thus, the diminishing demand for
completion accompanies stronger relational explanation, while
reduced feature-contribution change, reduced parent-set
reassignment, and the plateau in relation gain indicate that the
dictionary--graph pair is approaching a stable configuration over
the observed cycles. The final \(d_F\) and \(d_G\) values fall below
their prescribed tolerances, establishing operational joint
stability. Together, these trajectories support AG-SAE's central
training premise that induced structure guides targeted dictionary
refinement and the revised dictionary yields increasingly stable
and better-supported relations.

\begin{table}[!t]
\centering
\caption{
\textbf{Multi-parent topology spans all four dictionary banks, and Joint
validity changes by only \(0.75\) percentage points between two- and
three-parent sets.}
\textbf{A.} Cov.\ \% is the number of parented features divided by bank width,
and Multi \% is the number of multi-parent features divided by the number of
parented features.
\textbf{B.} Structures denotes the total number of recovered relations in each
arity cohort. Panel B applies the frozen report-only protocols of
Table~\ref{tab:real-llm-validation} to fixed arity-specific relation cohorts.
PSV, NR, and SV are evaluated on the same relations, while Joint reports the
percentage satisfying all three criteria. All rates are reported as percentages.
}
\label{tab:additional_graph_coverage_validity}

\begingroup
\footnotesize
\setlength{\tabcolsep}{4.5pt}
\renewcommand{\arraystretch}{1.03}
\setlength{\aboverulesep}{0.5pt}
\setlength{\belowrulesep}{0.5pt}

\begin{tabular*}{0.68\linewidth}{
    @{\extracolsep{\fill}}
    l
    rrrrr
    @{}
}
\toprule
\multicolumn{6}{c}{\textbf{A. Coverage by child width}} \\
\midrule
Width & Parented & Cov.\ \% & Single & Multi & Multi \% \\
\midrule
\(2{,}048\)  & 422   & 20.61 & 218   & 204 & 48.34 \\
\(4{,}096\)  & 1,197 & 29.22 & 847   & 350 & 29.24 \\
\(8{,}192\)  & 2,118 & 25.85 & 1,636 & 482 & 22.76 \\
\(16{,}384\) & 3,354 & 20.47 & 2,720 & 634 & 18.90 \\

\midrule
\multicolumn{6}{c}{\textbf{B. Validity by multi-parent arity}} \\
\midrule
Arity & Structures & PSV & NR & SV & Joint \\
\midrule
2 & 1,239 & 98.75 & 95.25 & 70.75 & 65.50 \\
3 & 431   & 95.25 & 88.50 & 79.25 & 64.75 \\
\bottomrule
\end{tabular*}

\endgroup
\end{table}

\paragraph{Final Topology Across Dictionary Scales and Parent-Set Arity.}
The resulting graph contains multi-parent relations across all four dictionary banks, comprising \(23.55\%\) of parented features overall (Table~\ref{tab:ag-sae-combined}). Graph roles separate systematically by scale, while relational depth extends beyond the four bank levels. Joint validity differs by only \(0.75\) percentage points between two- and three-parent sets. Panel A of Table~\ref{tab:additional_graph_coverage_validity} decomposes the \(23.55\%\) multi-parent share reported in Table~\ref{tab:ag-sae-combined} across the four dictionary banks. Among the \(7{,}091\) parented features, \(1{,}670\) (\(23.55\%\)) have multiple parents. Every bank contains between \(204\) and \(634\) multi-parent children, which constitute \(18.90\%\) to \(48.34\%\) of its parented features. Parented features likewise cover \(20.47\%\) to \(29.22\%\) of each bank across the eightfold range of dictionary widths. The within-bank multi-parent share increases monotonically as dictionary width decreases, reaching \(48.34\%\) in the \(2{,}048\)-feature bank.

Graph roles separate systematically with dictionary scale. The left panel of Figure~\ref{fig:graph_role_and_depth} shows that roots fall monotonically from \(27.29\%\) of the \(2{,}048\)-feature bank to \(5.94\%\) of the \(16{,}384\)-feature bank, while leaves rise monotonically from \(6.49\%\) to \(18.85\%\). Roots outnumber leaves by \(4.20\) to \(1\) in the narrowest bank, whereas leaves outnumber roots by \(3.17\) to \(1\) in the widest. The fraction of features that serve as a parent for at least one child follows the same ordering, decreasing from \(41.41\%\) in the narrowest bank to \(7.56\%\) in the widest. The induction procedure permits edges in either scale direction and within a bank, making this monotonic role separation an empirical property of the recovered relations. The resulting gradient is consistent with the increase in feature granularity with dictionary width documented in prior work \citep{bricken2023monosemanticity,bussmann2025matryoshka}. Narrow banks disproportionately contain explanatory roots, whereas wide banks disproportionately contain terminal refinements. Graph depth exceeds the limit of a strict four-level scale hierarchy. If every edge followed increasing dictionary width, any directed path beginning at a root could contain at most three edges. The right panel of Figure~\ref{fig:graph_role_and_depth} places \(394\) features beyond this bound and reaches a maximum depth of eight. The graph also contains \(2{,}381\) interior features with both a parent set and at least one child. Directed edges move from a narrower to a wider bank in \(71.80\%\) of cases, from a wider to a narrower bank in \(27.87\%\), and within the same bank in the remaining \(0.33\%\). These edge directions and paths extend the recovered organization beyond a four-level progression through dictionary scale.

The final graph contains \(1{,}239\) two-parent sets and \(431\) three-parent sets. Panel B of Table~\ref{tab:additional_graph_coverage_validity} reports arity-specific validity under the frozen report-only protocols of Table~\ref{tab:real-llm-validation}. Two-parent sets reach \(98.75\%\) PSV, \(95.25\%\) NR, and \(70.75\%\) SV, with a Joint rate of \(65.50\%\). Three-parent sets reach \(95.25\%\), \(88.50\%\), and \(79.25\%\) on the same metrics, with a Joint rate of \(64.75\%\). The two Joint rates remain nearly identical, differing by only \(0.75\) percentage points. The pooled Joint rate in Table~\ref{tab:real-llm-validation} is higher because it also includes the single-parent cohort. At both evaluated multi-parent arities, approximately two thirds of the sampled relations satisfy all three validity criteria.

\begin{figure}[!t]
    \centering
    \includegraphics[width=0.9\linewidth]{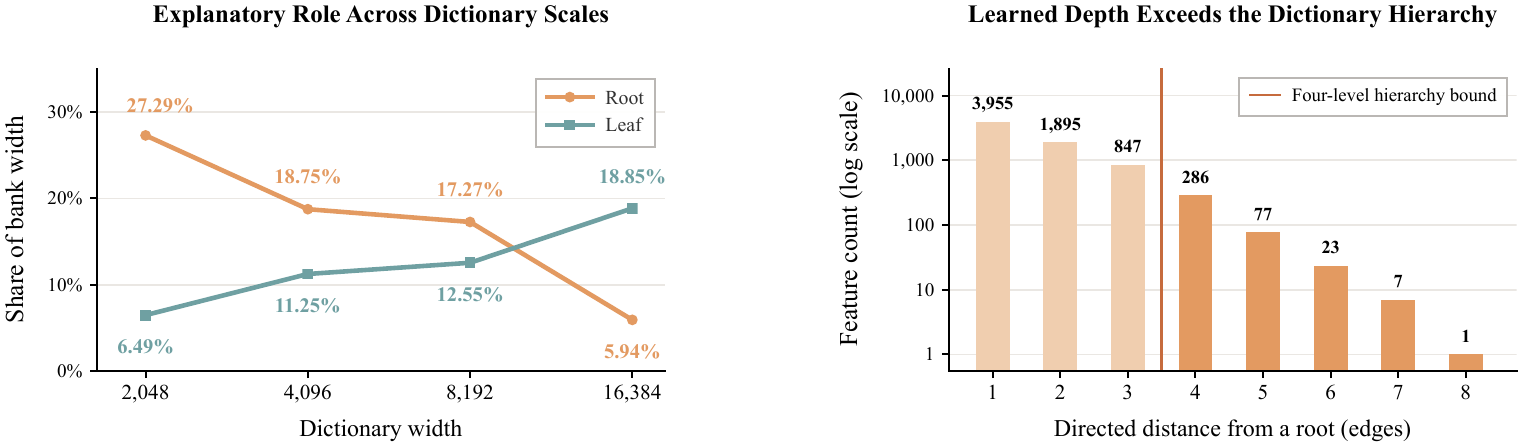}
    \caption{
    \textbf{Dictionary scale organizes graph roles, while graph depth extends
    beyond the four bank levels.}
    \textbf{Left.} Root and leaf shares within each dictionary bank. Roots fall
    from \(27.29\%\) to \(5.94\%\) of bank width, while leaves rise from
    \(6.49\%\) to \(18.85\%\).
    \textbf{Right.} A feature's graph depth is the maximum number of edges among
    all directed paths that begin at a root and end at that feature. The dashed
    line marks the maximum depth of three permitted by a strict progression
    from narrower to wider banks. The recovered graph places \(394\) features
    beyond this bound and reaches depth eight.
    }
    \label{fig:graph_role_and_depth}
\end{figure}

\subsection{Additional Analysis of the Learned Multi-Parent Relation}
\label{app:multi_parent_case_study}

We provide additional activation statistics and reconstruction
comparisons for the relation illustrated in
Figures~\ref{fig:real-case} and \ref{fig:case-coverage}.
The final AG-SAE trained on Gemma-2-2B layer-13 activations
from MiniPile assigns child \(c=7707\) the complete parent
set \(P=\{5963,14785,3242\}\).

\paragraph{Activation Coverage and Specificity.}
All activation statistics use the frozen dictionary with graph
routing disabled and activation thresholds fixed before report
evaluation. On the report split of \(4{,}910{,}199\) token
positions, the child activates at \(484\) positions.
Each parent covers \(92.77\%\)--\(97.73\%\) of these positions,
and all three parents jointly cover \(442/484=91.32\%\).
Conversely, only \(11.25\%\), \(38.44\%\), and \(4.52\%\) of
the activations of features \(5963\), \(14785\), and \(3242\),
respectively, coincide with child activation.
For the complete parent configuration, this proportion reaches
\(442/448=98.66\%\).
Thus, each parent individually covers most child activations
while responding over a substantially broader set of positions.
Their joint activation retains coverage of over nine tenths
of child-active positions and identifies the child's activation
regime with much greater specificity than any individual parent.

\begin{table}[t]
    \centering
    \caption{
    \textbf{Complete parent-set comparison for feature 7707.}
    The fit score used in PSV is one minus the squared prediction
    error for the child's decoder contribution, normalized by
    its total squared norm over the same \(484\) report positions
    where the child is active.
    Nonnegative coefficients are fitted separately for each
    parent set on a disjoint split and frozen for evaluation.
    The matched alternative is the strongest replacement set
    in the evaluated candidate pool.
    The null hypothesis has score \(0\).
    }
    \label{tab:case-parent-set}
    \begingroup
    \small
    \setlength{\tabcolsep}{6pt}
    \renewcommand{\arraystretch}{1.08}
    \begin{tabular}{@{}llr@{}}
        \toprule
        Comparison & Parent set & Held-out fit \(\uparrow\) \\
        \midrule
        \textbf{Complete learned set}
        & \(\{5963,14785,3242\}\)
        & \textbf{0.5497} \\
        Remove 5963
        & \(\{14785,3242\}\)
        & 0.4543 \\
        Remove 14785
        & \(\{5963,3242\}\)
        & 0.4409 \\
        Remove 3242
        & \(\{5963,14785\}\)
        & 0.3460 \\
        Best single parent
        & \(\{14785\}\)
        & 0.2452 \\
        Best matched alternative
        & \(\{5295,14785,3242\}\)
        & 0.4627 \\
        \bottomrule
    \end{tabular}
    \endgroup
\end{table}

\paragraph{Complete Parent-Set Comparisons.}
Table~\ref{tab:case-parent-set} compares the learned parent set
with its subsets and the strongest evaluated matched alternative.
The complete set achieves a held-out fit of \(0.5497\),
compared with \(0.2452\) for the best single parent and
\(0.4543\) for the best two-parent subset.
Removing any parent and refitting the remaining coefficients
reduces the report score by at least \(0.0954\).
Feature \(14785\) is particularly informative because its
responses to \emph{appeal} and \emph{appellant} overlap lexically
with the child. Removing it reduces the score to \(0.4409\),
demonstrating incremental predictive support beyond the other
two parents despite this lexical overlap.
At the same parent-set cardinality, replacing feature \(5963\)
with \(5295\) gives the strongest evaluated matched alternative,
with a score of \(0.4627\).
The learned set exceeds this alternative by \(0.0870\),
supporting the value of its specific parent combination.
Together, the subset and replacement comparisons establish
predictive contributions from every selected parent and favor
the learned combination over the evaluated alternatives.

\paragraph{Nonredundant Child Contribution.}
The child also supplies reconstruction value within the activation
regime shared with its complete parent set.
On the \(442\) report positions where the child and all three
parents are active, adding feature \(7707\) yields a normalized
innovation gain of
\(\mathcal S_{\mathrm{inn}}(c,P)=0.06177\).
The parent-only and parent-plus-child reconstructions use
the same report positions, with nonnegative coefficients fitted
separately on a disjoint split and frozen for evaluation.
Following the definition of \(\mathcal S_{\mathrm{inn}}\),
the reduction in reconstruction error is normalized by
the zero-reconstruction error on these positions.
The parent-set fit measures how well the parents predict
the child's decoder contribution, while the innovation gain
measures the child's additional contribution to reconstructing
the original activations.
The two comparisons therefore provide complementary support
for the learned relation.
The complete parent set supplies joint predictive support,
and the child retains a nonredundant representational
contribution even where their activations closely coincide.

\subsection{Results Across Models and Datasets}
\label{app:cross_setting_results}

AG-SAE recovers validated mixed topology under both an isolated
corpus change and a joint change in model family, layer, and
dictionary configuration.
The PubMed setting holds Gemma-2-2B, layer 13, and the four
dictionary scales fixed while replacing MiniPile with a
biomedical corpus. The Qwen setting retains MiniPile while
changing the model family, layer, number of dictionary scales,
and total dictionary capacity.

\begin{table}[t]
\centering
\caption{
\textbf{AG-SAE dictionary quality, graph composition, and relation
validity across models and datasets.}
Gemma/PubMed uses Gemma-2-2B layer 13 with four dictionary scales;
Qwen/MiniPile uses Qwen3.5-2B-Base layer 18 with two scales and
\(49{,}152\) features in total.
\textbf{A.} EV and \(L_0\) are report-only measurements with graph
routing disabled at each dictionary scale.
Cov.\ \% divides parented features by bank width.
Single \% and Multi \% divide the respective feature counts by
parented features; Root \% and Leaf \% divide their counts by
bank width. Roots have children and no parents; leaves have
parents and no children.
\textbf{B.} Validation follows the frozen report-only protocols of Table~\ref{tab:real-llm-validation}, Panel D. Structures denotes the total number of recovered relations in each cohort.
PSV, NR, SV, and Joint are evaluated under the same frozen report-only
protocol within each cohort, with Joint reporting simultaneous satisfaction
of all three criteria.
}
\label{tab:cross-setting-results}
\label{tab:cross-setting-composition}
\label{tab:cross-setting-validity}
\label{tab:pubmed_composition}
\label{tab:qwen_composition}
\label{tab:pubmed_validity}
\label{tab:qwen_validity}
\resizebox{0.9\linewidth}{!}{%
\begin{minipage}{\linewidth}
\begingroup
\footnotesize
\setlength{\tabcolsep}{2pt}
\renewcommand{\arraystretch}{1.15}
\setlength{\heavyrulewidth}{1pt}
\setlength{\aboverulesep}{2pt}
\setlength{\belowrulesep}{2pt}
\setlength{\abovetopsep}{0pt}
\setlength{\belowbottomsep}{0pt}
\setlength{\parskip}{0pt}

\begin{tabular*}{\linewidth}
{@{\extracolsep{\fill}}llrrrrrrrr@{}}
\toprule
\multicolumn{10}{@{}c@{}}{
    \textbf{A. Dictionary Quality and Graph Composition}
} \\
\midrule
& & \multicolumn{2}{c}{Dictionary}
& \multicolumn{6}{c}{Graph} \\
\cmidrule(lr){3-4} \cmidrule(l){5-10}
Setting & Width & EV & \(L_0\) & Parented & Cov.\ \%
& Single \% & Multi \% & Root \% & Leaf \% \\
\midrule
\multirow{4}{*}{\textbf{Gemma / PubMed}}
& \(2{,}048\)  & 0.660 & 50.29 & 328
& 16.02 & 50.61 & 49.39 & 34.28 & 4.74 \\
& \(4{,}096\)  & 0.686 & 50.12 & 1,082
& 26.42 & 70.15 & 29.85 & 23.07 & 7.08 \\
& \(8{,}192\)  & 0.710 & 49.74 & 2,349
& 28.67 & 74.37 & 25.63 & 17.19 & 14.64 \\
& \(16{,}384\) & 0.727 & 50.36 & 3,700
& 22.58 & 79.35 & 20.65 & 4.72 & 19.46 \\
\midrule
\multirow{2}{*}{\textbf{Qwen / MiniPile}}
& \(16{,}384\) & 0.771 & 49.74 & 3,873
& 23.64 & 54.66 & 45.34 & 20.06 & 12.34 \\
& \(32{,}768\) & 0.798 & 49.71 & 5,689
& 17.36 & 75.73 & 24.27 & 13.77 & 15.32 \\
\bottomrule
\end{tabular*}

\par\vskip 4pt\nointerlineskip
\begin{tabular*}{\linewidth}
{@{\extracolsep{\fill}}llrrrrr@{}}
\toprule
\multicolumn{7}{@{}c@{}}{
    \textbf{B. Relation Validity}
} \\
\midrule
Setting & Cohort & Structures & PSV & NR & SV & Joint \\
\midrule
\multirow{2}{*}{\textbf{Gemma / PubMed}}
& Single-parent & 5,608 & 99.63 & 96.38 & 73.75 & 71.75 \\
& Multi-parent  & 1,851 & 97.00 & 91.38 & 65.50 & 60.75 \\
\midrule
\multirow{2}{*}{\textbf{Qwen / MiniPile}}
& Single-parent & 6,425 & 99.38 & 93.13 & 65.50 & 63.13 \\
& Multi-parent  & 3,137 & 95.13 & 88.25 & 74.38 & 62.25 \\
\bottomrule
\end{tabular*}

\endgroup
\end{minipage}%
}
\end{table}

\subsubsection{Validated Mixed Topology Persists under a Corpus Change}
\label{app:cross_setting_pubmed}

AG-SAE reproduces the main mixed-topology pattern on PubMed
after four refinement cycles
(Table~\ref{tab:cross-setting-results}, Panel A).
Multi-parent relations comprise \(24.82\%\) of parented features,
compared with \(23.55\%\) on MiniPile
(Table~\ref{tab:ag-sae-combined}, Panel B), and maximum graph
depth reaches \(11\).
The topology spans all four dictionary banks, with multi-parent
shares decreasing from \(49.39\%\) to \(20.65\%\) as width grows,
following the same ordering as MiniPile
(Table~\ref{tab:additional_graph_coverage_validity}, Panel A).
Roots become less prevalent and leaves more prevalent with
increasing width, reproducing the scale-dependent role inversion.
Parented features cover \(16.02\%\)--\(28.67\%\) of each bank;
EV increases with dictionary width while native \(L_0\) remains
near \(50\).

Structural reliability also persists under the corpus change
(Table~\ref{tab:cross-setting-results}, Panel B).
Single-parent and multi-parent PSV reach \(99.63\%\) and
\(97.00\%\), within \(0.12\) and \(0.00\) percentage points
of their MiniPile counterparts
(Table~\ref{tab:real-llm-validation}, Panel D).
NR changes by less than one percentage point in both cohorts.
Joint validity reaches \(71.75\%\) for single-parent relations
and \(60.75\%\) for multi-parent relations.
Thus, more than three fifths of relations in each
cohort simultaneously satisfy predictive, nonredundancy,
and semantic criteria.
With the model, layer, dictionary scales, and evaluation
protocols held fixed, these results establish that broad
cross-domain heterogeneity is not required for AG-SAE to
recover validated multi-parent topology.

\subsubsection{Validated Mixed Topology Persists under a New Model and Dictionary Configuration}
\label{app:cross_setting_qwen}

AG-SAE recovers mixed topology and the same direction of
scale-dependent role separation in Qwen3.5-2B-Base under a
larger, two-scale dictionary configuration.
Total capacity increases from \(30{,}720\) to \(49{,}152\)
features, while the number of scales decreases from four
to two; the main setting uses four scales to match the
hierarchical baselines.
After eight refinement cycles, \(3{,}137\) of \(9{,}562\)
parented features have multiple parents, a share of
\(32.81\%\) compared with \(23.55\%\) in the main setting
(Table~\ref{tab:ag-sae-combined}, Panel B).
Within-bank multi-parent shares are \(45.34\%\) and \(24.27\%\)
in the narrower and wider banks, respectively.
Roots become less prevalent and leaves more prevalent in
the wider bank, preserving the role ordering observed on
Gemma with both MiniPile and PubMed.
EV increases with width, with native \(L_0\) near \(50\)
at both scales (Table~\ref{tab:cross-setting-results}, Panel A).

Relational depth extends well beyond the two dictionary scales.
A strict progression from the narrower to the wider bank
permits a maximum depth of one; the recovered graph reaches
depth nine, with \(2{,}880\) features exceeding that bound,
corresponding to \(30.1\%\) of all parented features.
This substantial population demonstrates learned relational
paths beyond the hierarchy defined by dictionary scale.

The recovered Qwen relations also retain joint predictive,
nonredundancy, and semantic validity
(Table~\ref{tab:cross-setting-results}, Panel B).
Single-parent and multi-parent Joint rates reach
\(63.13\%\) and \(62.25\%\), respectively, so more than
three fifths of relations in each cohort satisfy
all three criteria.
Together, the PubMed and Qwen results establish that validated
mixed topology, systematic role separation across dictionary
scales, and relational depth beyond those scales persist under
an isolated corpus change and under a joint change in model
family, layer, scale count, and total dictionary capacity.

\subsection{Additional Causal Intervention Results}
\label{app:causal_training_controls}

Matched training counterfactuals and paired statistics across concepts establish
that the stronger causal control reported in
Figure~\ref{fig:causal_intervention_triptych} is specific to full AG-SAE
training among the evaluated procedures. The gain holds simultaneously in
addition and ablation across most concepts and remains concentrated in the
leading feature at the same Top-1 activation rate.

\paragraph{Only full AG-SAE training strengthens bidirectional causal control.}
Starting from the same pre-refinement checkpoint, we compare full AG-SAE training with two conventional SAE continuation procedures. Normalized continuation retains decoder normalization, and free-scale continuation allows decoder scale to vary under the conventional SAE objective. All three branches use the same data order and training budget of \(40{,}000\) steps. The evaluation preserves the same \(13\) feature identities, fixed before the branch outcomes were assessed. For each feature, we subtract the matched random-control effect separately from its addition and ablation effects, compute the change from the shared checkpoint in each intervention direction, and retain the weaker change. A positive change therefore requires causal control to strengthen in both intervention directions. Table~\ref{tab:causal_training_counterfactuals} reports the mean change from the shared checkpoint in this weaker effect, its bootstrap confidence interval, and the number of features with positive change. Full AG-SAE training increases bidirectional causal control by \(0.867\) logits, with a \(95\%\) confidence interval of \([0.298, 1.506]\), and improves \(10/13\) features. Normalized and free-scale continuation produce mean changes of \(-0.215\) and \(-0.160\) logits and improve only \(3/13\) and \(5/13\) features. Relative to normalized continuation, full AG-SAE achieves a \(1.082\)-logit advantage with a \(95\%\) confidence interval of \([0.450, 1.894]\), favoring \(11/13\) features. The advantage also holds separately in both intervention directions, reaching \(0.794\) logits in addition and \(1.055\) logits in ablation. Under matched initialization, data order, training budget, and feature identities, only the full AG-SAE branch strengthens addition and ablation jointly. Additional conventional SAE optimization and allowing decoder scale to vary do not reproduce the gain.

\begin{table}[t]
\centering
\small
\caption{Matched changes in bidirectional causal control under full AG-SAE
training and two conventional SAE continuation controls.}
\label{tab:causal_training_counterfactuals}
\setlength{\tabcolsep}{5pt}
\renewcommand{\arraystretch}{0.95}
\begin{tabular}{lccc}
\toprule
Condition & Change & \(95\%\) CI & Wins \\
\midrule
Full AG-SAE
    & \(\mathbf{+0.867}\)
    & \(\mathbf{[0.298,\,1.506]}\)
    & \(\mathbf{10/13}\) \\
Normalized continuation
    & \(-0.215\)
    & \([-0.597,\,0.184]\)
    & \(3/13\) \\
Free-scale continuation
    & \(-0.160\)
    & \([-0.546,\,0.235]\)
    & \(5/13\) \\
\bottomrule
\end{tabular}
\end{table}

\begin{table}[t]
\centering
\small
\caption{Paired AG-SAE gains over Vanilla SAE across the \(26\) first-letter
concepts.}
\label{tab:causal_intervention_statistics}
\setlength{\tabcolsep}{4pt}
\renewcommand{\arraystretch}{0.95}
\begin{tabular}{lcccc}
\toprule
Metric & Gain & \(95\%\) CI & Wins & \(p\)-value \\
\midrule
Addition
    & \(\mathbf{+0.573}\)
    & \(\mathbf{[0.308,\,0.893]}\)
    & \(\mathbf{24/26}\)
    & \(1.05\times10^{-5}\) \\
Ablation
    & \(\mathbf{+0.685}\)
    & \(\mathbf{[0.260,\,1.240]}\)
    & \(\mathbf{21/26}\)
    & \(0.0025\) \\
Bidirectional
    & \(\mathbf{+0.674}\)
    & \(\mathbf{[0.281,\,1.210]}\)
    & \(\mathbf{21/26}\)
    & \(0.0025\) \\
Top-1 concentration
    & \(\mathbf{+0.681}\)
    & \(\mathbf{[0.259,\,1.241]}\)
    & \(\mathbf{20/26}\)
    & \(0.0094\) \\
\bottomrule
\end{tabular}
\end{table}

\paragraph{The causal advantage spans concepts and remains concentrated at Top-1.}
Table~\ref{tab:causal_intervention_statistics} aggregates the paired AG-SAE gains over Vanilla SAE across the \(26\) first-letter concepts. Gain denotes the mean difference in logits, and Wins counts concepts favoring AG-SAE. For Bidirectional, we first take the weaker of the addition and ablation effects within each method for each concept, and then compute the paired AG-SAE--Vanilla difference; Top-1 reports the single-feature cumulative
ablation comparison. Intervals are \(95\%\) paired bootstrap confidence intervals, and \(p\)-values use exact two-sided sign tests. AG-SAE improves addition by \(0.573\) logits across \(24/26\) concepts and ablation by \(0.685\) logits across \(21/26\) concepts. Both paired confidence intervals are strictly positive, with exact sign-test \(p\)-values of \(1.05\times10^{-5}\) and \(0.0025\). This weaker-direction comparison yields a \(0.674\)-logit bidirectional gain with a \(95\%\) confidence interval of \([0.281, 1.210]\) across
\(21/26\) concepts. The direction-specific results additionally show simultaneous strengthening of concept writing and erasure across most evaluated concepts. At Top-1, AG-SAE retains a \(0.681\)-logit advantage with a \(95\%\) confidence interval of \([0.259, 1.241]\) across \(20/26\) concepts, despite the identical \(85.1\%\) activation rate of the two methods. This locates the increased causal yield in the leading feature rather than in more frequent feature activation. Together, the matched continuation counterfactuals and paired statistics attribute the observed increase in causal control to the full AG-SAE training paradigm among the evaluated procedures.

\subsection{Additional Component Ablations}
\label{app:component_ablations}

We evaluate how Complete Parent-Set Induction and Absorption
Realignment recover reliable relations and translate structural
information into dictionary improvements.
Figure~\ref{fig:component-ablations} compares induction procedures
on a fixed dictionary, measures the additional gains from dictionary
optimization, and quantifies the reach and effectiveness of
absorption correction.
We then connect these component results to the full refinement
trajectories.

\begin{figure}[t]
    \centering
    \includegraphics[width=0.95\linewidth]
        {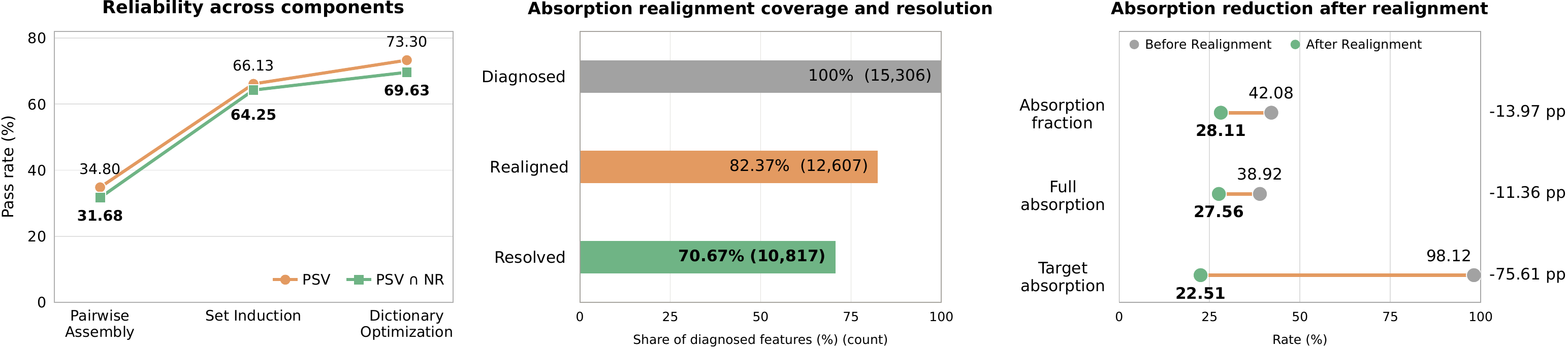}
    \caption{
    \textbf{Complete parent-set induction improves structural
    reliability, and realignment reduces feature absorption.}
    \textbf{Left.} Multi-parent PSV and
    \(\mathrm{PSV}\cap\mathrm{NR}\) under Pairwise Assembly,
    Complete Parent-Set Induction, and Dictionary Optimization.
    Dictionary Optimization denotes the full post-induction dictionary
    update, including Structure-Guided SAE Training and
    Mixed-Topology-Guided Dictionary Refinement, followed by one
    Full Re-Induction on the updated dictionary to obtain the evaluated graph.
    The first two configurations share a fixed dictionary
    and candidate space. Set Induction denotes Complete
    Parent-Set Induction.
    \textbf{Middle.} Realignment coverage and resolution over
    \(15{,}306\) benchmark-diagnosed absorption cases.
    Of these, \(12{,}607\) receive an Absorption Realignment update and
    \(10{,}817\) are subsequently resolved; percentages use all diagnosed
    cases as the denominator.
    \textbf{Right.} Absorption measurements before and after
    realignment. Gray and green mark the before and after
    values, respectively; reductions are in percentage points.
    }
    \label{fig:component-ablations}
\end{figure}

\paragraph{Complete Parent-Set Induction improves structural reliability.}
On the same fixed SAE dictionary and candidate space, Pairwise
Assembly constructs parent sets from independently selected
parent--child edges, while Complete Parent-Set Induction evaluates
each combination as an atomic structural hypothesis.
Identical EV, \(L_0\), and dictionary capacity isolate the
induction procedure's contribution from changes in representation
quality or sparsity.
Complete Parent-Set Induction raises multi-parent PSV from
\(34.80\%\) to \(66.13\%\) (\(+31.33\) percentage points) and
\(\mathrm{PSV}\cap\mathrm{NR}\) from \(31.68\%\) to \(64.25\%\)
(\(+32.57\) points; Figure~\ref{fig:component-ablations}, left).
At the same dictionary operating point, it more than doubles
the fraction of multi-parent relations combining predictive
parent support with a nonredundant child contribution,
directly supporting the complete parent set as the unit of
structural inference.
Dictionary Optimization further raises PSV to \(73.30\%\) and
\(\mathrm{PSV}\cap\mathrm{NR}\) to \(69.63\%\), gains of \(7.17\)
and \(5.38\) percentage points over induction alone.
These comparisons establish two complementary improvements.
Complete Parent-Set Induction extracts more reliable structure
from a fixed representation, and adapting the dictionary
improves relational reliability beyond that structural
selection step.

\paragraph{Absorption Realignment achieves broad coverage and substantial absorption reduction.}
The middle and right panels of Figure~\ref{fig:component-ablations}
report complementary measurements from a targeted experiment
evaluating Absorption Realignment. Starting from a fixed, fully trained SAE, we perform Complete
Parent-Set Induction once and apply Absorption Realignment once
using the resulting graph.
The SAEBench Feature Absorption
diagnostic~\citep{karvonen2025saebench}, which builds on the
first-letter evaluation of \citet{chanin2025absorption},
identifies an evaluation cohort of \(15{,}306\) absorption cases.
The benchmark diagnosis defines this cohort, while induction
and realignment use the SAE's representations and model-internal
quantities.
The middle panel shows that \(12{,}607\) cases (\(82.37\%\))
receive a realignment update and \(10{,}817\) (\(70.67\%\))
are subsequently resolved.
Both rates use the full diagnosed population, including cases
that receive no update or remain unresolved.
The operator thus reaches over four fifths of diagnosed cases
and resolves more than seven tenths of the full cohort.
The right panel reports the corresponding absorption reductions.
Under identical evaluation settings and the same frozen inference
threshold, Absorption Fraction falls from \(42.08\%\) to
\(28.11\%\) (\(-13.97\) percentage points), Full Absorption
from \(38.92\%\) to \(27.56\%\) (\(-11.36\) points), and
target absorption from \(98.12\%\) to \(22.51\%\)
(\(-75.61\) points).
The target correction thus accompanies improvements in both
aggregate absorption measurements.
EV changes by only \(-0.000040\), mean \(L_0\) by \(+0.0258\),
and the dead-feature fraction is unchanged.
Together with the frozen threshold, this stable operating point
makes the absorption reductions informative about feature
organization.
These results demonstrate that Absorption Realignment converts
the graph recovered by Complete Parent-Set Induction into
effective absorption correction, supporting its intended role
in reallocating parent-aligned content and AG-SAE's use of
learned structure to refine the dictionary.

\paragraph{Component effects align with the full training dynamics.}
The refinement trajectories in Figure~\ref{fig:graph_refinement_dynamics}
connect these component effects to the complete learning cycle. Residual Completion updates decline from \(12{,}245\)
in the first transition to \(12\) in the last, consistent with
diminishing demand for corrections to representational gaps
exposed by the recovered structure.
Feature-contribution change \(d_F\) and graph change \(d_G\)
fall to \(9.25\%\) and \(3.06\%\), respectively, while mean
held-out relation gain rises from \(4.56\%\) to \(9.12\%\).
The multi-parent population settles at \(1{,}670\) after
peaking at \(1{,}708\), showing that stronger relation support
accompanies structural consolidation without continued
expansion of the multi-parent population.
Together, these trajectories support progressive
dictionary--graph stabilization as completion demand
and parent-set reassignment diminish.
The fixed-dictionary comparison establishes the benefit
of complete parent-set inference, the isolated realignment
test establishes effective absorption correction, and the
dictionary-optimization gains and refinement trajectories
support their integration into an alternating process
that uses recovered relations to improve the dictionary.

\begin{figure}[t]
    \centering
    \includegraphics[width=0.37\linewidth]{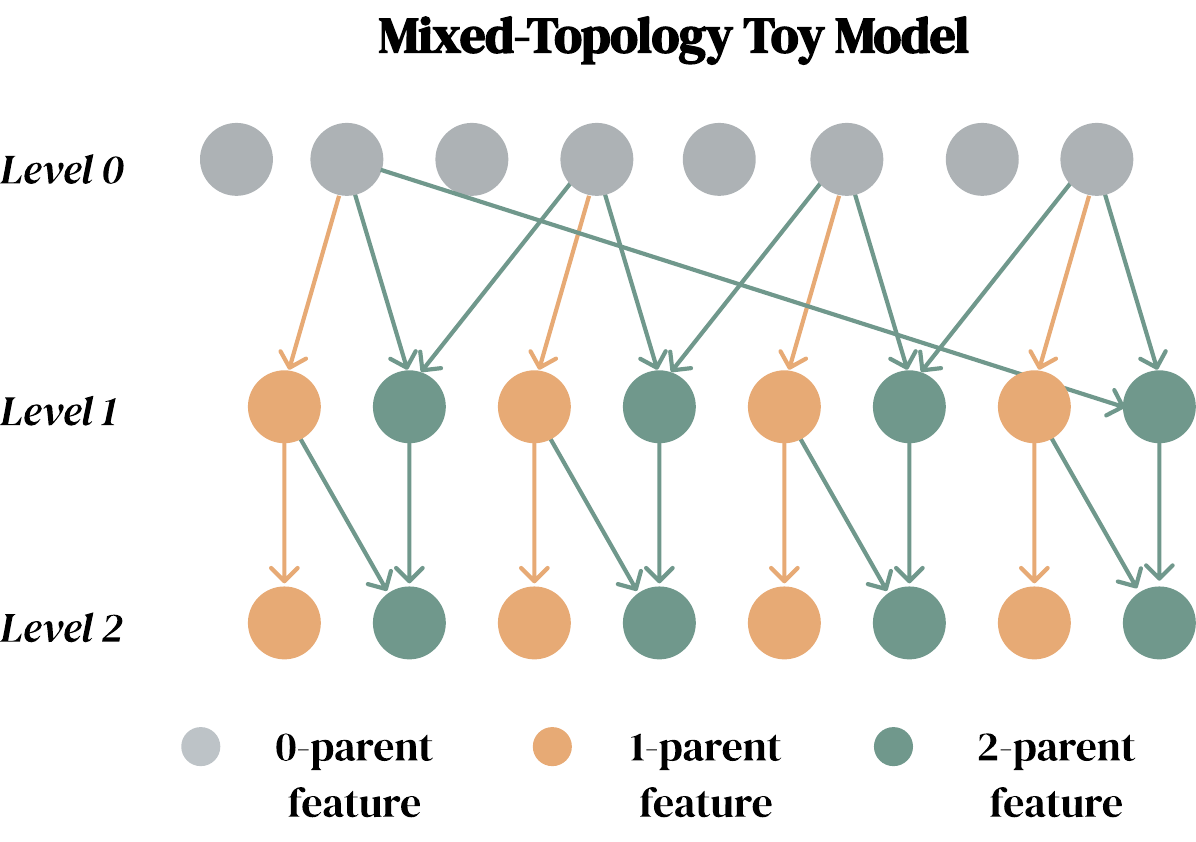}
    \caption{Ground-truth mixed topology of the toy model.
    Gray, orange, and green nodes have zero, one, and two immediate
    parents, respectively. Child activation requires all immediate
    parents to be active.}
    \label{fig:toy_model}
    \vspace{-15pt}
\end{figure}

\section{Experimental Setup and Hyperparameters}
\label{app:experimental_setup}

We specify the data construction, baseline configurations, and evaluation protocols underlying the reported results, together with the core training hyperparameters.

\subsection{Toy Model Construction and Recovery Evaluation}
\label{app:toy_model_setup}

Our toy model follows the SynthSAEBench framework~\citep{chanin2026synthsaebench}, generating sparse linear mixtures of 24 unit directions in 24 dimensions, with Gaussian-copula correlations among candidate activations and Gaussian magnitudes (mean 1.0, standard deviation 0.1) clipped at zero, with zero bias and no added noise. Feature directions are generated with graph-local correlations, and related active features share a common latent scale with feature-specific variation. Extending Matryoshka SAE's parent-conditioned toy model~\citep{bussmann2025matryoshka}, child activation requires all immediate parents to be active.
The three-level graph contains four isolated features, four roots, eight single-parent features, and eight two-parent features (Figure~\ref{fig:toy_model}).
Target activation probabilities for isolated features, roots, level-1 features, and level-2 features are 0.08, 0.20, 0.03, and 0.003, respectively, yielding an expected 1.384 active features per observation.
Before training, we specify 32 hard negatives comprising four correlation-only edges, 12 non-immediate ancestor edges, and the 16 singleton subsets of the eight true two-parent sets.
Each correlation-only pair targets a conditional activation probability of 0.70 for the associated isolated feature given child activation, compared with its marginal probability of 0.08.
Methods share the generated instance, 200,000 validation observations, and 1,000,000 independent test observations, with ground-truth directions and parent assignments reserved for evaluation.
Feature recovery uses Hungarian decoder matching with an absolute cosine threshold of 0.80.
Exact parent-set recovery additionally requires recovery of the child and all true parents and equality of the complete predicted and ground-truth immediate parent sets.
For methods producing a graph, hard-negative rejection measures the absence of forbidden edges and the nonselection of incomplete parent sets, complementing reconstruction and feature recovery with evaluation of the generating relations.

\subsection{LLM Activations and Baseline Comparisons}
\label{app:llm_baseline_setup}

\paragraph{Activation data and comparison setting.}
Following HSAE's choice of model, corpus, and reference dictionary sizes~\citep{luo2026hsae}, we use Gemma-2-2B~\citep{gemmateam2024gemma2} activations on MiniPile~\citep{kaddour2023minipile} from transformer block 13 (zero-based).
All methods share 100 million training and 5 million validation activations, extracted from sequences of up to 512 tokens with padding and the first eight positions excluded, and standardized per coordinate using training statistics.
The reference widths are 2,048, 4,096, 8,192, and 16,384, corresponding to independently reconstructing dictionaries for Vanilla SAE, HSAE, and AG-SAE, and feature groups within a 30,720-feature dictionary for Matryoshka SAE and Tree SAE.
BatchTopK~\citep{bussmann2024batchtopk} targets 50 active features per independent dictionary or full cumulative reconstruction; realized sparsity is reported in the main comparison table.

\paragraph{Vanilla SAE Post-hoc.}
We freeze the four independently trained SAEs and construct edges using the directed scaled masked cosine similarity of \citet[Appendix F]{bussmann2025matryoshka}, with the published threshold of 0.6.
The score computes activation cosine similarity on tokens where the child is active, scaled by the ratio of maximum child to maximum parent activation.
The candidate search space matches AG-SAE, including same-scale and all cross-scale candidates.
Activation-frequency ordering ensures acyclicity, and each child retains up to three highest-scoring qualifying parents.

\paragraph{Matryoshka SAE.}
Following \citet{bussmann2025matryoshka}, a single global BatchTopK operation selects activations, and equally weighted reconstruction losses over cumulative prefixes of 2,048, 6,144, 14,336, and 30,720 features update the shared parameters.
Graph extraction uses the same activation score and threshold between adjacent feature groups.
From the same trained dictionary, the DAG variant retains all qualifying parents, while the tree variant follows the paper's visualization projection and retains only the highest-scoring parent; children without qualifying parents remain unassigned.

\paragraph{Tree SAE.}
Following \citet{cao2026tree}, features occupy four privilege layers, with each feature assigned to the imaginary root or one parent in a lower privilege layer.
Training combines equally weighted cumulative reconstruction losses with layer-specific auxiliary reconstruction for inactive features, retaining dynamic reallocation of inactive children and root fallback.
The per-layer BatchTopK budgets are 34, 7, 5, and 4.
Parent gating follows sparsification, retaining a child only when its parent is active; removed activations are not replaced, so fewer than 50 features may remain active.

\paragraph{HSAE.}
We implement HSAE~\citep{luo2026hsae} with BatchTopK, initializing it from the four pretrained SAEs.
Each level reconstructs the same input.
Following the original method, each child selects one parent by maximum encoder cosine similarity in the preceding level; the lowest-scoring 20\% remain unassigned.
The objective retains the parent--children contribution constraint with the published weight of 0.01, and reconstruction replaces a parent's contribution with the sum of its children's contributions with probability 5\%.
Sparsification precedes substitution, and the constraint uses unperturbed contributions.
Assignments are updated every 5,000 steps and after training; joint optimization processes 20 million cached training activations.
Training budgets and optimization schedules are detailed in
Section~\ref{app:training_configuration}.

\paragraph{Additional model and corpus settings.}
Additional AG-SAE experiments use PubMed~\citep{nlm2025pubmed} with the same Gemma-2-2B layer and dictionary configuration, and Qwen3.5-2B-Base~\citep{qwen2026qwen35} at layer 18 on MiniPile with independent dictionaries of 16,384 and 32,768 features.

\subsection{Training Configuration}
\label{app:training_configuration}

For the main Gemma-2-2B comparison, Vanilla SAE~\citep{huben2024sparse}
receives 120k, 160k, 220k, and 300k updates for dictionary sizes of
2,048, 4,096, 8,192, and 16,384, respectively. AG-SAE uses the same
per-dictionary training budgets, including 25k, 25k, 50k, and 75k
initialization updates. AG-SAE's algorithmic stopping criterion is the
joint dictionary--graph stability rule defined in
Appendix~\ref{app:full_reinduction_stability}. For the main baseline
comparison, however, we additionally impose the matched training budgets
to prevent AG-SAE from receiving an advantage from additional optimization.
Accordingly, AG-SAE stops either when the joint stability criterion is
satisfied or when the matched training budget is exhausted.
Each dictionary is frozen upon reaching its update ceiling, and the
returned graph is freshly induced from the final dictionaries. Matryoshka SAE~\citep{bussmann2025matryoshka} and Tree SAE~\citep{cao2026tree} each receive 248k updates, matching the average of the four reference budgets weighted by their initial dictionary widths and corresponding to 507.904M activation presentations at batch size 2,048. HSAE~\citep{luo2026hsae} initializes from the fully trained Vanilla dictionaries and receives an additional 20k joint updates at batch size 1,024, with each shared batch updating all four dictionaries and providing 20.48M additional activation presentations per dictionary. All methods use the same cache of 100M training activations, with presentation counts including repeated cache passes. LLM baselines use AdamW with zero weight decay, gradient clipping at 1.0, and a peak learning rate of $3\times10^{-4}$, with 10\% linear warmup followed by cosine decay to zero within each training stage; HSAE's joint stage starts a fresh optimizer and learning-rate schedule. AG-SAE uses AdamW with zero weight decay, gradient clipping at 1.0, 500 warmup steps, and a 2,000-step structural-loss ramp per cycle, with learning rates of $2\times10^{-4}$, $1.5\times10^{-4}$, $10^{-4}$, $6\times10^{-5}$, and $3\times10^{-5}$ for cycles one through five and $3\times10^{-5}$ thereafter. All toy-model methods receive 40,000 Adam updates with $\beta=(0.5,0.9375)$ and batch size 256, corresponding to 10.24M training observations per method, using a learning rate of $3\times10^{-2}$ for the first 38,000 updates and $3\times10^{-3}$ for the final 2,000 updates; AG-SAE performs its full fixed-graph update, including
Structure-Guided SAE Training and Mixed-Topology-Guided Dictionary
Refinement, during these final 2,000 updates.

\vspace{-5pt}

\subsection{AG-SAE Hyperparameters}
\label{app:agsae_hyperparameters}

\begin{table}[t]
\centering
\caption{\textbf{Model-specific AG-SAE hyperparameters.}}
\label{tab:agsae_model_specific}
\footnotesize
\setlength{\tabcolsep}{3pt}
\renewcommand{\arraystretch}{0.88}
\setlength{\arrayrulewidth}{0.4pt}

\begin{tabular}{
!{\vrule width 0.8pt}
l|c|c
!{\vrule width 0.8pt}
}
\noalign{\hrule height 0.8pt}
\rule{0pt}{2.3ex}\textbf{Hyperparameter}
    & \textbf{Gemma}
    & \textbf{Qwen} \\[0.1ex]
\noalign{\hrule height 0.8pt}
$\lambda_{\mathrm{cov}}$
    & $0.006$ & $0.012$ \\
\hline
$\lambda_{\mathrm{inn}}$
    & $0.0015$ & $0.003$ \\
\hline
Initial $\lambda_{\mathcal I}$
    & $0.05$ & $0.015$ \\
\hline
LR, cycles 1--5
    & $(2,1.5,1,0.6,0.3)\times10^{-4}$
    & $(3,2.2,1.5,0.9,0.6)\times10^{-4}$ \\
\hline
LR, cycles 6--8
    & $3\times10^{-5}$
    & $6\times10^{-5}$ \\
\hline
LR, cycles 9--16
    & $3\times10^{-5}$
    & $3\times10^{-5}$ \\
\hline
Updates/cycle, 9--16
    & $40{,}000$
    & $40{,}000$ \\
\hline
Min. residual gain
    & $10^{-4}$ & $2\times10^{-4}$ \\
\hline
Min. residual validation/fitting ratio
    & $0.25$ & $0.50$ \\
\noalign{\hrule height 0.8pt}
\end{tabular}
\vspace{-10pt}
\end{table}

For both Gemma and Qwen, we restrict parent sets to $|P|\leq 3$, retrieve 24 candidates, and use an exact-search pool of 12, with $\tau_{\mathrm{cmp}}=10^{-3}$, $\tau_{\mathrm{cov}}=0.40$, and $\tau_{\mathrm{inn}}=3\times10^{-4}$. The representational-support threshold $\tau_{\mathrm{rep}}$ is calibrated separately for each model, dictionary width, and parent-set cardinality using six random and six wrong-parent controls, a $0.99$ control quantile, a $0.001$ margin, and a $0.01$ floor. The reconstruction coefficients $\boldsymbol\beta$ used for child innovation are fitted by unregularized NNLS, requiring 128 fitting and 64 validation events for both child and joint support. The nonnegative prediction coefficients
$\boldsymbol\alpha^{c,P}$ used by
$\mathcal L_{\mathrm{rep}}$ are likewise fitted by unregularized NNLS
on child-active samples at the beginning of each fixed-graph training
phase. AG-SAE training and refinement use $\lambda_{\delta}=10^{-6}$, $\lambda_{\mathrm{sp}}=0.1$,
$\lambda_{\mathrm{rep}}=0.02$, $\lambda_{\mathrm{cmp}}=1$, $m=0.005$,
$T=0.02$, $\eta=0.15$, and smooth-gate temperature $0.5$ for both models;
model-specific parameters are summarized in Table~\ref{tab:agsae_model_specific}. The overall structural weight $\lambda_{\mathcal I}$ targets a $3\%$ structural gradient contribution with a 2,000-update ramp. Training uses AdamW with batch size 2,048, zero weight decay, gradient clipping at $1.0$, and target $L_0=50$ per bank. Residual Completion requires 64 fitting and 64 validation events,
uses a reuse cosine-similarity threshold of $0.95$, calibrates newly
initialized feature thresholds at the $0.995$ activation quantile,
and caps capacity growth at $5\%$. Joint stability uses $\gamma_G=0.05$ and $\gamma_F=0.10$ over 262,144 validation samples. Training is capped at 16 cycles, followed by fresh re-induction of the terminal graph.

\vspace{-5pt}

\subsection{Structural and Semantic Evaluation}
\label{app:structural_semantic_evaluation}

\paragraph{Structural evaluation.}
Recovered relations are evaluated using the same scoring rules and thresholds across methods, with dictionaries, graphs, and thresholds frozen before report evaluation. For the main Gemma-2-2B/MiniPile comparison, we use fixed evaluation cohorts of \(800\) single-parent, \(400\) two-parent, and \(400\) three-parent relations from each frozen output for every arity supported by the method. Each cohort is sampled once and reused for PSV, NR, SV, and Joint. Joint pools these fixed cohorts over all relation arities supported by the method. Thus, methods supporting both single- and multi-parent relations are evaluated over \(1{,}600\) sampled relations for Joint, while single-parent-only methods use the \(800\)-relation single-parent cohort.
Nonnegative coefficients are fitted on a separate split
and held fixed for evaluation. PSV compares the complete
parent set with the null, proper subsets, and evaluated
matched alternatives of the same cardinality in predicting the
child's decoder contribution on held-out child-active samples.
The score is one minus squared prediction error divided by
the child's contribution energy; passing requires a margin
of at least $10^{-3}$ over the strongest competitor.
NR compares reconstruction using the parents alone with
reconstruction using the parents and child on jointly active
samples; passing requires an error reduction of at least
$10^{-3}$ after normalization by the zero-reconstruction error.

\vspace{-5pt}

\paragraph{Semantic evaluation.}
AG-SAE and all baselines use identical prompts, evidence counts,
and sampling procedures within each task, with Qwen3-30B-A3B
as the evaluator~\citep{qwen2025qwen3}.
Following SAEBench~\citep{karvonen2025saebench}, feature
descriptions are generated independently from ten
highest-activation and five activation-weighted contexts,
then frozen before relation evaluation.
Single-parent SV adopts HSAE's subset-or-specialization
criterion~\citep{luo2026hsae}, using descriptions and
representative activating contexts. Multi-parent SV evaluates
joint semantic coverage from the descriptions, requiring
every parent to be relevant and rejecting mere co-occurrence.
Judges receive no method identities, feature identifiers,
or quantitative scores, and semantic outputs never enter
induction. Within each arity, SV reports the fraction of sampled relations satisfying the semantic criterion, and Joint reports the fraction satisfying PSV, NR, and SV simultaneously.

\noindent
Single-parent semantic evaluation uses the following prompt.

\begin{lstlisting}[style=semanticprompt]
[System]
We are studying neurons in a sparse autoencoder (SAE). Each neuron activates on specific words, substrings, or concepts in short documents, with the activating text marked by <<...>>.

Determine whether the Child neuron's activating concept is a subset or a more specific version of the Parent neuron's activating concept.

Output exactly:
HaveRelationship: Yes or No
Confidence: High, Medium, or Low

Do not include any other text.

[User]
Parent neuron
Frozen semantic label: {parent_label}
Top-activation text 1: {parent_top_1}
Top-activation text 2: {parent_top_2}
Activation-weighted text: {parent_weighted}

Child neuron
Frozen semantic label: {child_label}
Top-activation text 1: {child_top_1}
Top-activation text 2: {child_top_2}
Activation-weighted text: {child_weighted}
\end{lstlisting}

\noindent
Multi-parent semantic evaluation uses the following prompt,
with parent labels presented in randomized order.

\begin{lstlisting}[style=semanticprompt]
[System]
You judge whether a set of Parent SAE feature labels jointly forms a semantic envelope that covers one Child feature label.

Answer Yes when the combined semantic scope of the Parents covers the stable core meaning of the Child, so the Child is a natural semantic specialization, intersection, or combination within the Parent meanings. Different Parents may cover different aspects of the Child. No individual Parent is required to contain the entire Child.

Before deciding, check every Parent one by one. Every Parent must be directly relevant to a core semantic aspect of the Child. If even one Parent is irrelevant, answer No. Do not answer Yes merely because one Parent already covers the Child while ignoring the other Parents.

Also answer No when the Parents and Child are merely related or commonly co-occur, when the combined Parent scope misses an essential part of the Child meaning, or when the Child is not a coherent specialization, intersection, or combination of the Parent meanings.

Judge only the frozen feature labels. Do not use graph structure, activation strength, quantitative scores, feature IDs, or method identity. Parent order has no meaning.

Return exactly one JSON object:
{"joint_coverage":"Yes","reason":"brief phrase"}
or
{"joint_coverage":"No","reason":"brief phrase"}

[User]
Parent feature labels:
{numbered_parent_labels}

Child feature label:
{child_label}
\end{lstlisting}

\subsection{Causal Intervention Protocol}
\label{app:causal_intervention_protocol}

We compare AG-SAE and Vanilla SAE on all 26 first-letter concepts, using the task introduced by \citet{chanin2025absorption} and SAEBench's single-feature sparse probing procedure~\citep{karvonen2025saebench}.
For each dictionary, probing on the training split selects one leading feature per letter; feature identities, inference thresholds, and addition scales are frozen before evaluation.
Both methods use the same 416 held-out prompts, with eight target-letter and eight non-target-letter examples per concept.
Following SAE feature-intervention studies~\citep{templeton2024scaling}, we intervene at the queried token's layer-13 residual stream in Gemma-2-2B, expressing decoder vectors in the original residual-stream coordinates.
Addition applies the selected decoder vector scaled by its median nonzero activation on target-class training examples to non-target prompts; ablation subtracts its observed thresholded contribution from target prompts.
These interventions evaluate causal control at each feature's native contribution scale.
Controls use random directions orthogonal to the target decoder and unrelated letter features, each matched to the actual intervention's L2 norm on each prompt.
The target margin is the target-letter logit minus the mean of the other 25 letter logits.
Effects are margin increases for addition and margin decreases for ablation relative to the same prompt without intervention, including zero-activation ablation cases.
To measure concentration, we cumulatively ablate the leading one to five features in the frozen sparse-probe ranking, computing all contributions from the same unmodified activation.
Single-feature ablation and the cumulative curve are evaluated in separate
bfloat16 forward passes, batching the four intervention controls and the six
prefixes \(k=0,\ldots,5\), respectively; the Top-1 curve value is therefore an
independently executed estimate of the same leading-feature intervention.
Causal yield per active feature is the mean ablation effect divided by the mean number of active intervened features.
We average prompts within each letter and compare methods across letters using paired bootstrap 95\% confidence intervals and exact two-sided sign tests.

\end{document}

%% file: math_commands.tex
\usepackage{amsmath,amsfonts,bm}

\def\eqref#1{equation~\ref{#1}}

\def\1{\bm{1}}

\DeclareMathAlphabet{\mathsfit}{\encodingdefault}{\sfdefault}{m}{sl}
\SetMathAlphabet{\mathsfit}{bold}{\encodingdefault}{\sfdefault}{bx}{n}

